\documentclass{article}

\usepackage{microtype}
\usepackage{graphicx}
\usepackage{caption}
\usepackage{subcaption}
\usepackage{booktabs} % for professional tables
\usepackage{listings}

\usepackage{hyperref}

\usepackage[accepted]{icml2024}

\usepackage{amsmath}
\usepackage{amssymb}
\usepackage{mathtools}
\usepackage{amsthm}

\usepackage[capitalize,noabbrev]{cleveref}

\theoremstyle{plain}
\newtheorem{theorem}{Theorem}[section]

\newtheorem{lemma}[theorem]{Lemma}

\theoremstyle{definition}

\newtheorem{assumption}[theorem]{Assumption}
\theoremstyle{remark}

\usepackage[textsize=tiny]{todonotes}

\icmltitlerunning{Offline Training of Language Model Agents with Functions as Learnable Weights}

\usepackage{color}
\usepackage{colortbl}
\usepackage{float}
\usepackage{caption}
\usepackage[algo2e,ruled,noend,vlined,linesnumbered]{algorithm2e}
\usepackage{multirow}
\usepackage{bbding}
\usepackage{pifont}
\usepackage{wrapfig}

\newcommand{\cmark}{\ding{51}}%
\newcommand{\xmark}{\ding{55}}%
\definecolor{Gray}{gray}{0.9}
\definecolor{burntorange}{rgb}{0.8, 0.33, 0.0}

\DeclareMathOperator*{\argmin}{arg\,min}

\definecolor{codegreen}{rgb}{0,0.6,0}
\definecolor{codegray}{rgb}{0.5,0.5,0.5}
\definecolor{codepurple}{rgb}{0.58,0,0.82}
\definecolor{backcolour}{rgb}{0.95,0.95,0.92}

\lstdefinestyle{mystyle}{
    backgroundcolor=\color{backcolour},   
    commentstyle=\color{codegreen},
    keywordstyle=\color{magenta},
    numberstyle=\tiny\color{codegray},
    stringstyle=\color{codepurple},
    basicstyle=\tiny\ttfamily,
    breakatwhitespace=false,         
    breaklines=true,                 
    captionpos=b,                    
    keepspaces=true,                 
    showspaces=false,                
    showstringspaces=false,
    showtabs=false,                  
    tabsize=2
}

\usepackage{etoc}
\etocdepthtag.toc{mtchapter}
\etocsettagdepth{mtchapter}{subsection}
\etocsettagdepth{mtappendix}{none}

\begin{document}
\twocolumn[
\icmltitle{Offline Training of Language Model Agents with Functions as Learnable Weights}



\icmlsetsymbol{equal}{*}

\begin{icmlauthorlist}
\icmlauthor{Shaokun Zhang}{equal,psu}
\icmlauthor{Jieyu Zhang}{equal,uw}
\icmlauthor{Jiale Liu}{psu}
\icmlauthor{Linxin Song}{usc}
\icmlauthor{Chi Wang}{msr}
\icmlauthor{Ranjay Krishna}{uw}
\icmlauthor{Qingyun Wu}{psu}
\end{icmlauthorlist}

\icmlaffiliation{psu}{Pennsylvania State University}
\icmlaffiliation{uw}{University of Washington}
\icmlaffiliation{usc}{University of Southern California}
\icmlaffiliation{msr}{Microsoft Research}

\icmlcorrespondingauthor{Qingyun Wu}{qingyun.wu@psu.edu}

\icmlkeywords{Machine Learning, ICML}

\vskip 0.3in
]

\printAffiliationsAndNotice{\icmlEqualContribution} % otherwise use the standard text.
\begin{abstract}
Researchers and practitioners have recently reframed powerful Large Language Models (LLMs) as \emph{agents}, enabling them to automate complex tasks largely via the use of specialized functions.
To facilitate the development of LLM agents, we present a novel paradigm of training LLM agents without modifying the LLM weights, which is particularly useful when the LLMs are difficult or inaccessible for modifications.
Inspired by how humans continuously forge tools to adapt to real-world tasks, rather than change our biological structure to fit a static set of tools, we propose to progressively forge agent's functions to better solve the downstream tasks instead of modifying the LLM weights. 
By treating the functions as learnable `agent parameters' and leveraging the fundamental idea of model training in artificial intelligence, we develop AgentOptimizer that employs the LLM to update agents' functions and devise an \emph{agent training} algorithm with two strategies, roll-back, and early-stop, to streamline the training process. 
With extensive experiments, we showcase that the agent training paradigm could significantly improve the performance of representative LLM agents in various downstream tasks. 
We also study the behavior of the agent training regarding aspects like the learning curve and domain transferability. We have integrated our method into \href{https://github.com/microsoft/autogen/blob/main/notebook/agentchat_agentoptimizer.ipynb}{AutoGen} library.  
\end{abstract}

\section{Introduction}
\label{intro}

Reframing Large Language Models (LLMs) as agents has ushered in a new paradigm of automation—one where LLMs can utilize existing functions~\footnote{Note that the literature has used the term `functions' to sometimes refer to tools or other actions.}  to accomplish complex tasks~\cite{xi2023rise, wang2023survey,yao2022react,wang2023voyager,humphreys2022data, ALFWorld20}. 
For example, LLM agents, armed with a function to `search over Wikipedia' can answer knowledge questions; agents with the ability to `issue SQL queries' can search large databases.
Functions allow LLMs to access external knowledge sources~\cite{peng2023check}, offload numerical computation~\cite{wu2023empirical}, search the internet~\cite{shi2017world}, and much more~\cite{qin2023tool}.

\begin{figure}[!t]
\begin{center}
\centerline{\includegraphics[width=0.8\columnwidth]{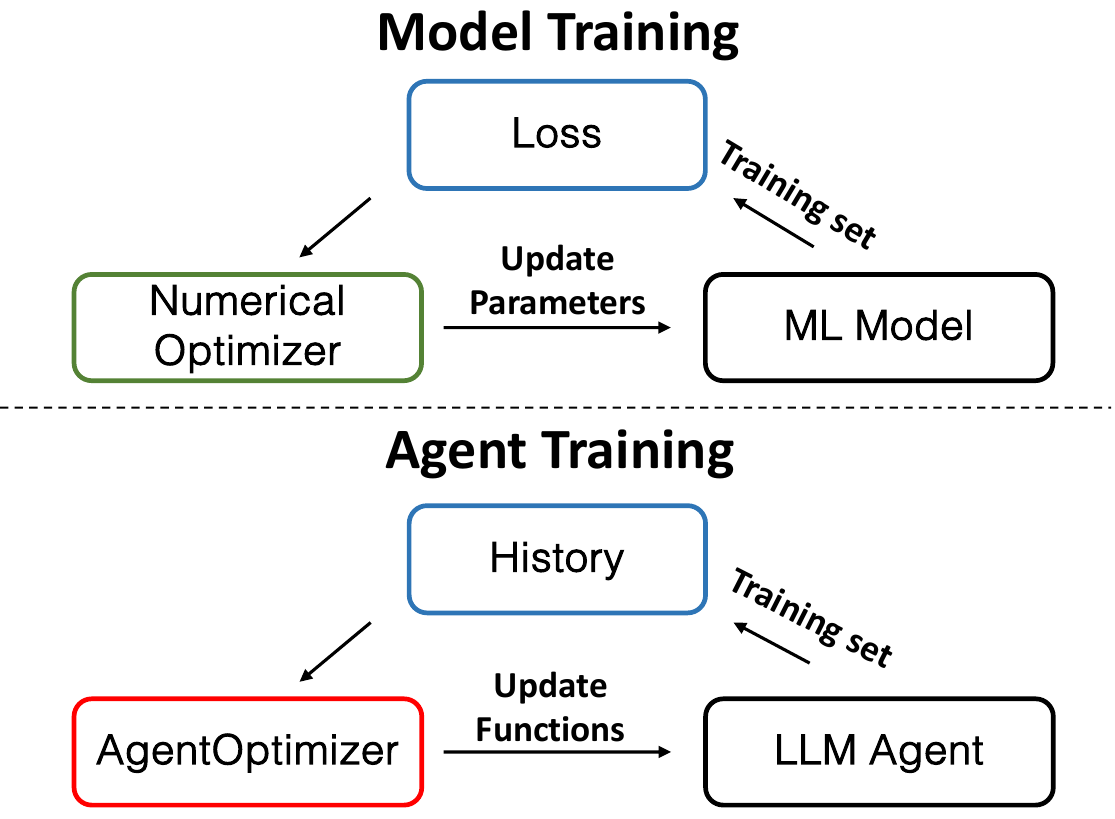}}
\caption{The comparison between model training and agent training. In model training, \emph{numerical optimizers}~\cite{ruder2016overview} such as SGD and Adam optimize the model weights according to the loss on the training set. In contrast, agent training iteratively updates the agents' functions according to the execution history using the proposed \emph{AgentOptimizer}.}
\label{demo}
\end{center}
\vspace{-0.4in}
\end{figure}

To enable LLM agents with useful functions, users need to first manually create functions that would be helpful for specific downstream tasks. This curation process may require many iterations and, therefore, be time-consuming. 
Since LLMs are black boxes, researchers have found that LLMs unexpectedly fail to utilize certain kinds of functions~\cite{qin2023tool}. In response, researchers have tried to improve the underlying LLM's capability of using existing functions by finetuning the LLM with ground truth function calls~\cite{qin2023toolllm,zeng2023agenttuning}. This finetuning process requires large computing resources. Worse, it limits which LLMs can be used since many LLM models are proprietary. 

Inspired by the fact that human-made tools become an extension of the human user ~\cite{botvinick1998rubber} and how humans forge tools to best adapt to real-world tasks, rather than change the biological structure of the human to fit a static set of tools, we propose a new \emph{agent training} paradigm that `forge' the functions for the LLM agent to use to best adapt to the training tasks.
It addresses both aforementioned challenges at the same time since it does not require finetuning the underlying LLM and could start with an empty set of functions. While the LLM's parameters are never updated, its functions are optimized to maximize the agent's ability to solve tasks.

In specific, we draw an analogy between traditional model training and our agent training (Figure~\ref{demo}). (1) Instead of updating model parameters, our training process updates the functions for LLM agents, viewing them as the agent's `trainable parameters'. (2) Instead of a loss calculated over a training set, our training process uses the agent's execution history and performance on training tasks
as the basis for updating the agent's functions.
Since we operate in the space of functions, numeric optimizers such as SGD or Adam are not applicable. Instead, we develop \textbf{AgentOptimizer}, which leverages the LLM to update the agent's functions based on the execution history and the agent-generated / ground truth answer from the current epoch.
In particular, the AgentOptimizer is instructed to progressively update the current function set by performing one of the predefined function manipulation actions (add, revise, and remove), rather than regenerate the whole function set at each optimization step.  

With AgentOptimizer, the overall workflow of LLM agent training is as follows: given a training set and an empty function set, at each epoch, we first evaluate the agent system against the training set and collect the execution history as well as the agent-generated / ground truth answers, then we feed this information to the AgentOptimizer to perform an optimization step to update the current function set.
To avoid potential performance degradation caused by function updates, we introduce two simple strategies: roll-back and early-stop.
The former is to withdraw the current function updates if the performance over the training set is degraded and roll back to the previous status, while the latter is to early terminate the optimization process when a certain number of consecutive optimization steps do not introduce any performance boost over the training set.

We conducted extensive empirical evaluations on three distinct tasks: mathematical reasoning (MATH)~\cite{hendrycks2021measuring}, tabular processing (TabMWP)~\cite{lu2023dynamic}, and general real-world problems (GAIA)~\cite{mialon2023gaia}. We trained two typical agent systems, GPT-4+ agent~\cite{openai2023gpt} and ReAct agent~\cite{yao2022react}, using the agent training method.
For the MATH dataset, agent training resulted in an obvious performance improvement in almost all cases. For more realistic and complex tasks GAIA and TabMWP, agent training led to an average performance improvement of 6\% and 8.5\% in GPT-4+ agent and ReAct agents, respectively.
We also perform ablation to demonstrate the efficacy of different components of the agent training method. In addition to ablation, we analyzed its extension to large-scale training and its transferability across different domains.

Our contributions are summarized below:
\begin{itemize}
    \item \textbf{(Paradigm)} Inspired by the fundamental idea of model training in machine learning, we introduce a tailored paradigm for training LLM agents without modifying the LLMs to build specialized LLM agents for a given application: we establish analogies between the learnable parameters inherent in traditional models and the operational functions of LLM agents, as well as between the models' loss functions and the agents' execution history over the training set, to craft a training regime that enhances the LLM agents' capabilities;
    \item \textbf{(Methodology)} To realize this paradigm, we propose the AgentOptimizer as an alternative to numeric optimizers used in traditional model training. It is designed to operate in the space of the operational functions of LLM agents via the exceptional capabilities of LLMs. Based on the AgentOptimizer, we develop a training algorithm with two additional techniques (roll-back and early-stop) that streamline the training process; 
    \item \textbf{(Experiments)} We conduct extensive experiments on three distinct tasks in training two typical agent systems, the GPT-4+ agent and the ReAct agent, to showcase the advantage of the proposed paradigm. We also provide ablation studies and detailed analysis to understand the behavior of the agent training.
\end{itemize}

\section{Methodology}

We begin by defining notations and setting up the research problem. 
We use $S_{\mathcal{F}}$ to denote any LLM agent system with function set $\mathcal{F}=\{f_{1}, ..., f_{n} | \forall i \in [n], f_{i} \in \mathcal{V}\}$. $f_{i}$ denotes the $i_{th}$ function that can be used by agent system $S_{\mathcal{F}}$ in the function space $\mathcal{V}$.

Throughout this work, we assume black-box LLMs such as ChatGPT~\cite{openai2022gpt} in the form of LLM as services.
Given any task with training data $\mathcal{D}_{train}$ and test data $\mathcal{D}_{test}$, the goal of this study is to find a set of functions $\mathcal{F^{*}}$ that could improve the LLM agent's expected performance on unseen test data $\mathcal{D}_{test}$. To put it more formally, 

\begin{equation}
\label{eq:final_goal}
\mathcal{F^{*}} = \argmin_{\mathcal{F}\subset\mathcal{V}}E[Loss(S_{\mathcal{F}}, \mathcal{D}_{test})],
\end{equation}
where $Loss(S_{\mathcal{F}}, \mathcal{D}_{test})$ measures the average loss of the agent $S_{\mathcal{F}}$ on test data $\mathcal{D}_{test}$.  
In the context of agent training throughout this paper, loss is defined as the rate of failed problem-solving attempts using agent systems.

However, the test set and its distribution are not available. In traditional machine learning model training, it is a common practice to assume that the distribution of the training and test data are the same or similar. 
While this assumption doesn’t always hold, in machine learning practice, training loss is used ubiquitously as the primary metric for parameter selection as a compromise solution.

Following the same spirit, we also employ training data as a proxy for test data. Then optimizing the functions of the language agent in function space by minimizing the loss of training data. This approach allows us to approximate the performance of the language agent on unseen test data, i.e.,
$
\hat{\mathcal{F}} = \argmin_{\mathcal{F}\subset\mathcal{V}}Loss(S_{\mathcal{F}}, \mathcal{D}_{train}),
$
where $\hat{\mathcal{F}}$ is approximation of $\mathcal{F^{*}}$.

\subsection{The AgentOptimizer}

To obtain $\hat{\mathcal{F}}$, it is critical to develop an optimizer tailored for agent training: it should be capable of updating current functions according to the agent system's performance on the training set. In contrast to traditional model training where the optimization is conducted over a numeric model parameter space and derivative-based optimizers can be applied with a loss of choice, agent training aims to search for the optimal set of functions for the agent system and therefore existing numeric model optimizers are not applicable.

Considering these, we propose the \textbf{AgentOptimizer} which leverages LLMs' exceptional capability of understanding and generating language to iteratively update the current set of functions as an optimizer.
Specifically, at each optimization step, we prompt the AgentOptimizer with the current status of the agent system and its execution history and performance on the training set and instruct it to update the functions of the agent system.
Intuitively speaking, this iterative optimization paradigm could lead to the identification of optimal functions in a large language space, analogous to iteratively performing gradient descent when training traditional machine learning models.

\paragraph{The input to the AgentOptimizer.}

We use $H$ to denote the information used to prompt LLMs, which mainly comprises the following two parts:
1) The execution history of the agents in solving each problem of the training set, including the details of how the agent uses current functions and 
2) the final performance over the training data.
In addition, we include the current set of functions associated with the agent system as input.
This information is necessary for the AgentOptimizer to be aware of the current state of the agent system and accordingly suggest function updates.

\begin{algorithm}[htb!]\small
\captionsetup{font=small} 
\SetNoFillComment
\SetAlCapNameFnt{\small}
\SetAlCapFnt{\small}
\renewcommand{\baselinestretch}{0.5} 

\caption{Progressive Function Update (\texttt{AgentOptimizer.step})}
\label{alg:2}
\KwIn{Functions to be optimized $\mathcal{F}^{0}$, historical information $H$}
\KwOut{Updated agent functions $\tilde{\mathcal{F}}$}
\textbf{Initialization:} $\tilde{\mathcal{F}} \gets \mathcal{F}^{0}, t \gets 0$ 

\While{$t < \texttt{MAX\_NUM}$}{
\texttt{Action} $\gets$ \texttt{LLM}($\mathcal{F}^{t}$, $H$)

\If{\texttt{Action} = \texttt{TERMINATE}}{~~\texttt{break}}
\Else{~~\tcp{add/revise/remove function}  $~~\mathcal{F}^{t+1} = \texttt{Action}(\mathcal{F}^{t}) \\\tilde{\mathcal{F}} \gets \mathcal{F}^{t+1} $}
$~~t \gets t+1$
}
\textbf{Return} $\tilde{\mathcal{F}}$
\end{algorithm}

\paragraph{Progressive function update.}
Given the inputted information $H$, a naive way of updating the current functions is to instruct the LLM to regenerate the whole function set to replace the existing one.
However, such an aggressive optimization step is unwise since it overwrites all existing functions, discards useful functions already established, and requires the LLM to generate multiple functions in a single shot.
In contrast, we propose to progressively update the functions via predefined \emph{actions} within each optimization step. 
In particular, we adopt four actions:
1) \emph{add\_function}: add one new function that may be useful; 2) \emph{revise\_function}: revise one existing function; 3) \emph{remove\_function}: remove one existing function; and 4) \emph{TERMINATE}: terminate the function update process.
Except for the TERMINATE action, all the actions 
require certain \emph{action arguments} as input; For example, to perform the \emph{add\_function} action, the LLM needs to generate the name, description, code, etc. of the function as the action arguments so that when executed, this action will add the new function to the function set. More details are presented in Appendix~\ref{optimizer-actions}.
At each time step, the AgentOptimizer is prompted to choose one action until the maximum number of actions is reached or the AgentOptimizer chooses \texttt{TERMINATE}, and the resulting function set will be returned.
The overall procedure of progressive function update is shown in Algorithm~\ref{alg:2}. 
\subsection{Agent Training}

With the AgentOptimizer, we then present the overall agent training procedure.
In practice, the function updates suggested by the AgentOptimizer may cause performance degradation, since the LLM is not the oracle for updating the functions. 
Therefore, we propose two simple strategies for the training procedure to avoid performance degradation.

\paragraph{Roll-back.}

To avoid performance degradation after function updating, we employ a simple yet effective strategy named roll-back. Specifically, at each optimization step, if the latest function update leads to a performance drop on the training set, the AgentOptimizer will withdraw this update. Moreover, considering the fact that LLMs are shown to be able to recognize patterns from in-context demonstrations~\cite{wei2022chain}, we also record the failed updated function and the corresponding performance in a list (Line 11 of Algorithm~\ref{alg:1}).  This list will be used as the prompt for the next function generations from AgentOptimizer. We expect that LLMs could use the historical failure information to generate better functions. The list will be cleared after achieving performance improvement.

\paragraph{Early-stop.}

In extreme situations, the optimization process may stuck, and rollback repeats without improving the performance. In this case, it is wise to terminate the optimization process, and we employ an early-stop strategy: the optimization process will be automatically terminated after $C$ consecutive optimization steps without performance improvement over the training set.

\paragraph{Overall agent training algorithm.}
The pseudocode of the agent training is shown in Algorithm~\ref{alg:1}.
The agent training process takes as input the following parameters: training data $\mathcal{D}_{train}$, agent system $S$, maximum training epoch $E$, and the early-stop threshold $C$.
After an initialization step, which sets the initial functions list $\mathcal{F}_{0}$ and initial historical information $H_{0}$ to empty sets, the algorithm proceeds as follows: at each iteration $i$,  the AgentOptimizer optimizes the functions list $\mathcal{F}_{i}$ to obtain $\mathcal{F}_{i+1}$ based on historical information $H_{i}$.  The updated function set is then evaluated on the training set to obtain evaluation information for the next epoch of training. The training procedure terminates when the maximum epoch or early-stop threshold is reached.

\begin{algorithm}[!htb]\small
\SetNoFillComment
\SetAlCapNameFnt{\small}
\SetAlCapFnt{\small}
\renewcommand{\baselinestretch}{0.5}
\caption{Agent Training}
\label{alg:1}
\KwIn{Training Data $\mathcal{D}_{train}$, agent system $S$, max training epoch $E$, early-stop threshold $C$}
\KwOut{Enhanced agent system $S_{\mathcal{\hat{F}}}$}
\textbf{Initialization:} $i \gets 0, r \gets 0, H_{0} \gets \varnothing, \mathcal{F}_{0} \gets \varnothing$.

\While{$i < E$}{
\If{$H_{i} \neq \varnothing$}{$~~\mathcal{F}_{i+1}$ = \texttt{AgentOptimizer.step}($\mathcal{F}_{i}$, $H_{i}$)}
\Else{$~~\mathcal{F}_{i+1} \gets \mathcal{F}_{i}$}
$H_{i+1}$ = \texttt{Eval}($S_{\mathcal{F}_{i+1}}$, $\mathcal{D}_{train}$)

\If{$H_{i+1}.loss$ $<$ $H_{i}.loss$}{$~~H_{i}$.\texttt{fail\_record} $\gets \emptyset, \\ ~~\mathcal{\hat{F}} \gets \mathcal{F}_{i+1}, i \gets i +1$, $r$ $\gets$ $0$}
\Else{$~~H_{i}$.\texttt{failure\_record} $\gets$ ($\mathcal{F}_{i+1}, H_{i+1}.loss)$\\
$r$ $\gets$ $r$ $+ 1$

~~\If{$r$ $>$ $C$}{\tcp{Early stop} ~~\texttt{Break}}
} 
}
\textbf{Return} $S_{\mathcal{\hat{F}}}$ 

\end{algorithm}
\vspace{-0.2in}

\section{Experiments}
We conduct experiments to prove the superiority of the proposed method. 
We begin by providing the experimental settings in Section~\ref{sec_expsetup}. 
We then evaluate the agent training method on three datasets to verify its effectiveness in Section~\ref{section_main_result}. 
Finally, we perform in-depth investigations in the last three sections to provide a better understanding of the proposed agent training method.

\begin{table*}[t]
\renewcommand{\arraystretch}{0.2} % Increase row spacing by a factor of 1.5
\setlength{\tabcolsep}{2pt} % Reduce column spacing
\scalebox{0.78}{
\renewcommand{\arraystretch}{1.3}
\begin{tabular}{ccccccccccccccc}
\toprule[1.5pt]
\multicolumn{1}{c}{\multirow{2}{*}{\textbf{Data types}}} & \multicolumn{2}{c}{P.Algebra} & \multicolumn{2}{c}{Algebra} & \multicolumn{2}{c}{I.Algebra} & \multicolumn{2}{c}{Geometry} & \multicolumn{2}{c}{C.Probability} & \multicolumn{2}{c}{Precalculus} & \multicolumn{2}{c}{N.Theory} \\ \cline{2-15} 
\multicolumn{1}{c}{}                               & \multicolumn{1}{c}{Train} & \multicolumn{1}{c}{Test} & \multicolumn{1}{c}{Train} & \multicolumn{1}{c}{Test} & \multicolumn{1}{c}{Train} & \multicolumn{1}{c}{Test} & \multicolumn{1}{c}{Train} & \multicolumn{1}{c}{Test} & \multicolumn{1}{c}{Train} & \multicolumn{1}{c}{Test} & \multicolumn{1}{c}{Train} & \multicolumn{1}{c}{Test} & \multicolumn{1}{c}{Train} & \multicolumn{1}{c}{Test}         \\ \hline
GPT-4+ Agent w/o Agent Training                         &   60.0\%              & 78.8\%                &    55.0\%            & \textbf{66.3\% }            &30.0\%                  & 30.0\%                 &  30.0\%              & 40.0\%               &    65.0\%              & 72.5\%                &     5.0\%             &  32.5\%                &     70.0\%            &   56.3\%            \\
\rowcolor[gray]{0.9} GPT-4+ Agent w/ Agent Training                          &    \textbf{65.0\% }             & \textbf{82.6\% }               &    \textbf{65.0\% }           &      65.0\%        &   \textbf{40.0\% }               & \textbf{38.8\% }               &   \textbf{40.0\% }            &  \textbf{42.5\% }             & \textbf{65.0\% }                & \textbf{76.3\% }              &    \textbf{10.0\% }             &  \textbf{35.0\% }              & \textbf{80.0\% }               & \textbf{67.5\% }              \\
ReAct Agent w/o Agent Training                           &     55.0\%             &    87.5\%            &     55.0\%           &  \textbf{83.8\% }           &    25.0\%               &50.0\%                 &   5.0\%             &  53.8\%             &  45.0\%                &      73.8\%         &      5.0\%            &  53.8\%               &  75.0\%               & 68.8\%              \\
\rowcolor[gray]{0.9} ReAct Agent w/ Agent Training                            & \textbf{55.0\%}                & \textbf{87.5\%}             &      \textbf{60.0\%}         &  82.5\%           &    \textbf{35.0\%}              &    \textbf{51.3\%}            &         \textbf{15.0\%}      &   \textbf{58.8\%}           &        \textbf{50.0\%}         &     \textbf{78.8\%}          &          \textbf{10.0\%}       &        \textbf{62.5\%}        &     \textbf{75.0\%}           & \textbf{72.5\%}              \\
 \bottomrule[1.5pt]
\end{tabular}
}
\vskip -0.09in
\caption{Train/Test accuracy of GPT-4+/ReAct agents with/without agent training on MATH datasets. We show the accuracy of each data type. We can observe that agent training could lead to an obviously better performance for both two agent systems in most cases. 
}
\label{tab:math_main}
\end{table*}

\begin{table*}[!t]
\centering
\setlength{\tabcolsep}{23.1pt}
\scalebox{1.0}{
\renewcommand{\arraystretch}{0.9} % Adjust the factor to decrease the space
\begin{tabular}{lcccc}
\toprule
\multicolumn{1}{c}{\multirow{2}{*}{\textbf{Method}}} & \multicolumn{2}{c}{GAIA} & \multicolumn{2}{c}{TabMWP} \\ 
\cmidrule(lr){2-3}
\cmidrule(lr){4-5}
                        & Train     & Test    & Train      & Test   \\ \hline
GPT-4+ Agent w/o Agent Training                & 10.0\%     &  16.0\%  & 30.0\%      & 51.0\%      \\
\rowcolor[gray]{0.9}  GPT-4+ Agent w/ Agent Training   & \textbf{30.0\%}     &  \textbf{23.0\%}  & \textbf{66.7\%}      & \textbf{56.0\%}      \\ 
ReAct Agent w/o Agent Training                  &  20.0\%    &  12.0\%  & 63.3\%      &  59.0\%     \\
\rowcolor[gray]{0.9} ReAct Agent w/ Agent Training                & \textbf{40.0\%}     &  \textbf{18.0\%}  & \textbf{73.3\%}      & \textbf{70.0\%}      \\
\midrule[.1em]
\end{tabular}}
\vskip -0.09in
\caption{Train/Test accuracy of GPT-4+/ReAct agents with/without agent training on the GAIA and TabMWP datasets. 
We can observe that agent training can lead to greater performance for both GPT-4+ and ReAct agents on both two datasets.}
\label{tab:gaiatab_main}
\end{table*}

\begin{table*}[!t]
\centering
\setlength{\tabcolsep}{4pt} % Reduce column 
\renewcommand{\arraystretch}{0.9}
\begin{tabular}{cccc}
\toprule
\textbf{Method}         & Number Theory & Intermediate Algebra & Counting and probability \\ \hline
 
No Agent Training &  56.3\%        &   30.0\%               &  72.5\%                     \\
Agent Training w/o Roll-back \& Early-stop      &  63.8\%         &   36.3\%               &  72.5\%                     \\
Agent Training w/o Progressive Function Update     &  60.0\%       &   28.8\%               &  70.0\%         \\

\rowcolor[gray]{0.9} Agent Training (Ours) & \textbf{67.5\% }               &  \textbf{38.8\% }                    &  \textbf{76.3\%}                        \\ 
\bottomrule
\end{tabular}
% \vskip -0.09in
\caption{We take the training of the GPT-4+ agent as an example and perform ablation to investigate the effect of different components of the agent training method on three data types of the MATH dataset.}
\label{tab:abla}
\end{table*}

\subsection{Experimental Setup}
\label{sec_expsetup}
\paragraph{Evaluation tasks and metrics.}
To evaluate the effectiveness of the proposed agent training, we conducted experiments on three distinct tasks: \emph{Mathematical Reasoning}, \emph{Tabular Processing}, and \emph{General Real-World Tasks}. 
Due to the high cost of OpenAI models, it is impractical to evaluate the method on the complete datasets and therefore we subsample data from these datasets for training and testing, following the same settings as previous works~\cite{yuan2023craft, wu2023empirical}. The number of training examples is set according to the LLM's context limit.

\textbf{(1) Mathematical reasoning:}
Following a similar setting with \cite{yuan2023craft}, we use a subset of MATH datasets~\cite{hendrycks2021measuring} to evaluate the LLM agent's performance in addressing mathematical problems. 
For each data type (7 in total), we randomly choose 20 training examples and 80 test examples, and report the accuracy of each data type respectively. 

\textbf{(2) Tabular processing:} 
The TabMWP~\cite{lu2023dynamic} dataset evaluates agents in processing structured data in tables, where each data sample contains one table and one question in natural language. We randomly sampled 100 test examples and 10 training examples. We measured the model performance using accuracy based on the exact match.

\textbf{(3) General real-world tasks:} The GAIA dataset~\cite{mialon2023gaia} is dedicated to evaluating the LLM agents in solving unambiguous real-world questions. From its public subset, we randomly select 10 questions for training and 100 questions for testing and report the correct rate as suggested in the original paper. 

\paragraph{Agent systems employed.}
We employ the proposed agent training method to train two typical LLM agent systems:

\textbf{(1) GPT-4+ agent:}  GPT-4+ agent essentially is GPT-4 with function call and code interpreter.
The GPT-4 plays the role of making reasoning decisions, while the code interpreter executes code and function calls suggested by the GPT-4.

\textbf{(2) ReAct agent:}  The ReAct agent~\cite{yao2022react} generates both reasoning traces and task-specific actions in an interleaved manner to solve tasks. In our evaluations, we optimized the ReAct agent to improve its actions that may be taken at each action step after a reasoning process.

For both the GPT-4+ agent and ReAct agent, we initialize them with Python as the initial function that can execute the Python code suggested by the LLMs

\paragraph{Models.}
For the more challenging tasks on the MATH and GAIA datasets, we used GPT-4-1106-preview for both AgentOptimizer and LLMs agents. For the easier task TabMWP, we chose to use GPT-3.5-turbo-1106 to construct LLMs agents and GPT-4-1106-preview to construct the AgentOptimizer. This was done to better visualize the improvement brought by the agent training and did not sacrifice the conclusions obtained from the experiments.

\subsection{Main Results}
\label{section_main_result}

\paragraph{Mathematical reasoning.}

We first evaluated the performance of GPT-4+ agent and ReAct agent on the MATH dataset, as well as their performance after agent training on train/test splits, as shown in Table~\ref{tab:math_main}. Across seven data types, we observed that agent training led to better performance on the test set in most cases (11 out of 14). Additionally, training performance improved in almost all cases, while in the remaining cases, it remained the same.  
Our results indicate that agent training could produce functions useful for unseen test tasks. Interestingly, for counting and probability problems, when training GPT-4+ agent, the training performance remains the same while test performance improves from 72.5\% to 76.3\%. 
This suggests that in specific situations, even if the generated functions do not lead to performance improvement on the training set, they are helpful for the unseen test data.
\vspace{-0.15in}

\paragraph{Tabular processing and general real-world tasks.}
We then perform evaluations on Tabular Processing tasks TabMWP~\cite{lu2023dynamic} and general real-world tasks GAIA~\cite{mialon2023gaia}, as shown in Table~\ref{tab:gaiatab_main}. Our observations indicate that agent training led to performance improvements for both two agent systems. Since these two datasets are more realistic and complex than MATH, our results demonstrate that agent training can generate general and usable functions that increase agents’ realistic task-solving capabilities, indicating that agent training is practically useful to some extent

\subsection{Ablation and Analysis}
\label{sec_analysis}

\subsubsection{Ablation}

We conducted ablation experiments to evaluate the effectiveness of two different components of the agent training method: (1) roll-back \& early-stop, and (2) progressive function updating. 
To achieve this goal, we chose three data types of MATH that resulted in the largest performance improvements in training GPT-4+ agent: number theory, intermediate algebra, and counting and probability
\footnote{Pre-Algebra was not selected due to its similarity to Intermediate Algebra, despite having the same performance improvements as counting and probability.}.
Specifically, to investigate (1), we removed roll-back and early-stop and trained the agent until reaching the maximum epoch number. The agent status will not roll-back when the training performance drops. For (2), we replaced the progressive function update with a one-step function generation, which directly prompted the GPT-4 in AgentOptimizer to generate the functions at each epoch. We also showed the origin GPT-4+ agent performance without agent training.

As shown in Table~\ref{tab:abla}, the performance greatly dropped if either one of them was removed. Another interesting observation is that agent training without progressive function update even exhibited worse performance than the origin GPT-4+ agent without agent training. This scenario proves that prompting LLMs to generate functions is non-trivial. A bad function generation method may even lead to a negative effect. Therefore, a carefully designed function generation algorithm is desirable.

\begin{table*}[!t]
\centering
\setlength{\tabcolsep}{10pt} % Reduce column 
\renewcommand{\arraystretch}{0.8}
\begin{tabular}{ccccc}
\toprule[1.5pt]
\textbf{Method}                     & MATH - Train & MATH - Test & TabMWP - Train & TabMWP - Test \\ \hline
% LATM~\cite{cai2023large}   &        &       &                &               \\
CREATOR~\cite{qian2023creator}    &   N/A     &  75.0\%           &       N/A         &  30.0\%             \\ 
CRAFT~\cite{yuan2023craft}   &    50.0\%    &  73.8\%     &     38.0\%           &    38.5\%           \\
\hline
\rowcolor[gray]{0.9} GPT-4+ Agent w/ Agent Training  &  \textbf{60.0}\%  & 66.25\%      &   \textbf{66.6\%}             &  \textbf{56.0\%} \\ 
\rowcolor[gray]{0.9} ReAct Agent w/ Agent Training &  \textbf{60.0\%}  &  \textbf{77.5\%}  &  \textbf{73.3\%}              &  \textbf{70.0\%}             \\
\bottomrule[1.5pt]
\end{tabular}
\vskip -0.09in
\caption{
The comparisons between the trained agent systems with two typical tool-creation methods on MATH and TabMWP datasets. 
CREATOR doesn't involve a training stage so the training performance is unavailable.
The results indicate that both GPT-4+ agent and ReAct agent trained with our method outperform tool-creation methods in most cases.
}

\label{compare_with_tools}
\vspace{-0.2in}
\end{table*}

\subsubsection{Learning curve}

\begin{figure}[!htb]
    % \vskip -0.5in
    \begin{subfigure}{0.48\linewidth}
    \centering
    \includegraphics[width=\linewidth]{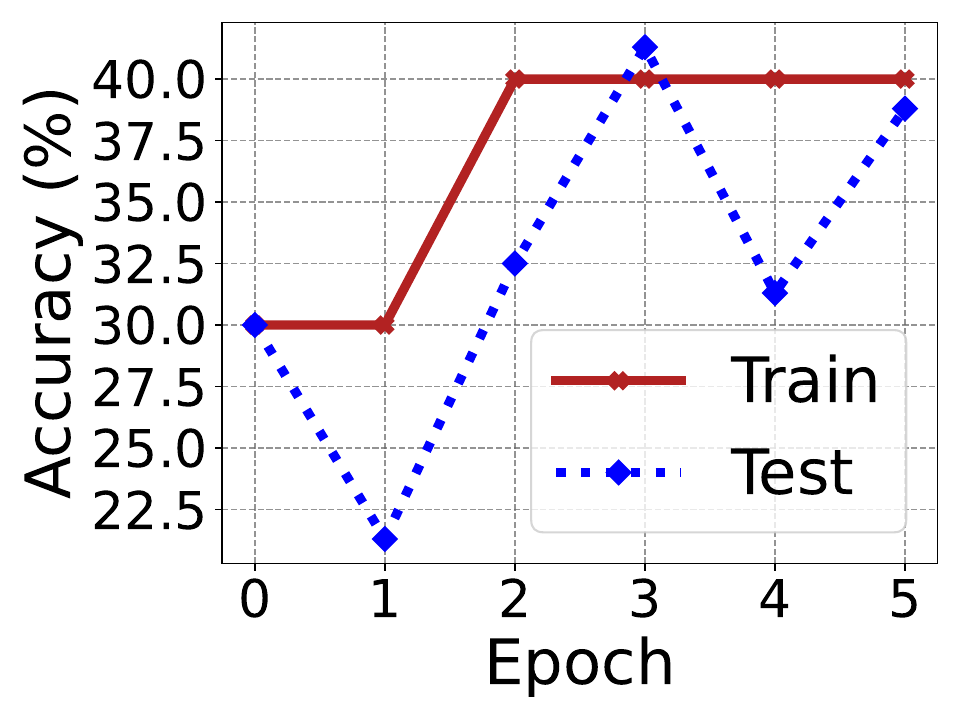}
    \caption{Positive - I.Algebra}\label{learning_curve_a}
    \end{subfigure}%
    \centering
    \begin{subfigure}{0.48\linewidth}
    \centering
    \includegraphics[width=\linewidth]{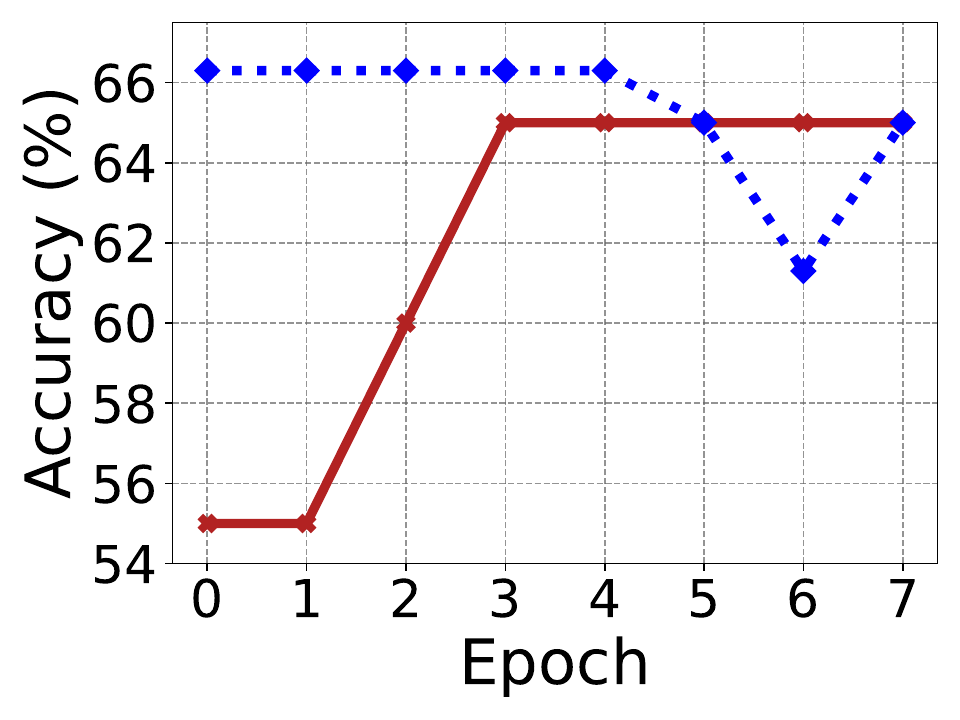} 
    \caption{Negative - Algebra}\label{learning_curve_b}
    \end{subfigure}%
    \caption{On the MATH dataset, we visualize the changes in train/test performance across epochs when training a GPT-4+ agent. For analysis purposes, we select one data type where the training does improve the test performance (Positive) and another that does not (Negative).}
    \label{learning_curve_ana} 
    % \vspace{-0.5in}
\end{figure}

For analysis purposes, we visualize the learning curve when training GPT-4+ agent in solving mathematical problems in Figure~\ref{learning_curve_ana}.
According to the types of experiment results, i.e., whether test performance improves (positive) or not (negative), we choose two data types, the only data type that failed to improve the test performance (Algebra) and one similar data type with the failed one that successfully obtained test performance improvement  (Intermediate Algebra).
Regarding the positive results on Intermediate Algebra in Figure~\ref{learning_curve_a}, 
we observe that when the optimization starts, the test performance is better than it was at the start time in most epochs, and the test performance is positively correlated with the training performance in general. These scenarios provide evidence to demonstrate the effectiveness of agent training. 
However, we also notice that the highest test performance is not at the last epoch where the algorithm terminates. 
To some extent, it represents GPT-4+ agents overfitting to the training set and suffering from a test performance drop while the training performance remains the same. 
Regarding the negative results on Algebra in Figure~\ref{learning_curve_b}, we get a similar observation that the test performance drops while the training performance remains the same. 
We also found the scenario that the test performance remains the same while training performance improves, indicating that sometimes the generated tool may be not general enough to be useful but would not harm the performance in solving tasks.

\subsubsection{Domain  Transferability}

% \vspace{-0.1in}
\begin{figure}[H]
\label{fig:demo}
\begin{center}
\centerline{\includegraphics[width=0.59\columnwidth]{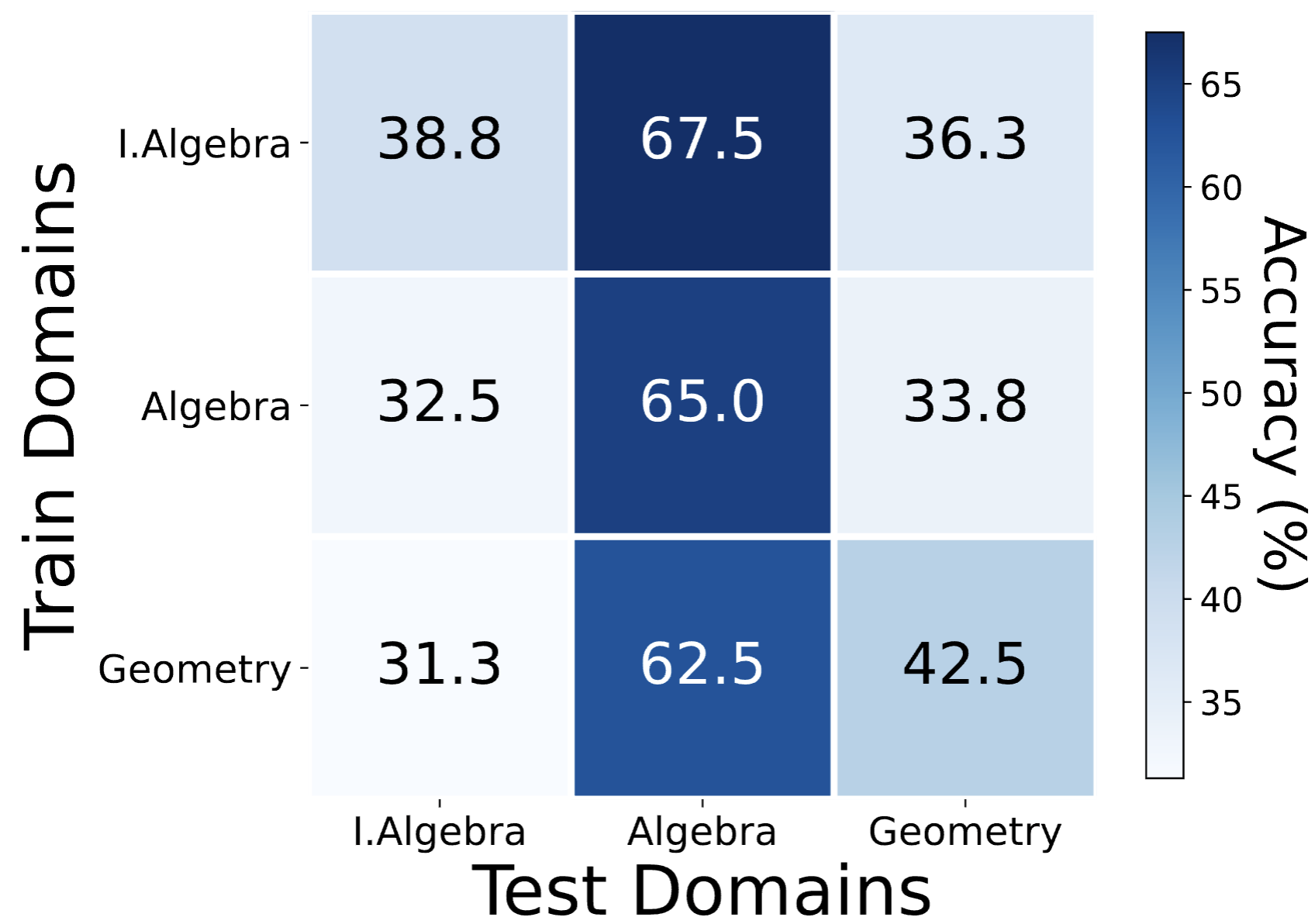}}
\caption{
To investigate the domain transferability of the agent training method, we show test performances of three different data types of the MATH dataset after training with different domains.
}
\label{fig_ood}
\end{center}
\vspace{-0.4in}
\end{figure}
We then investigate the generalization and transferability~\cite{zhou2022domain} of the agent training method when the test data and training data are not sampled from the same domain. 
We use three data types in MATH: algebra, intermediate algebra, and geometry.
We intend to choose two data types that have similar distributions (algebra and intermediate algebra) and another data type that should have the largest semantic distance with the algebra and intermediate algebra (geometry). 
We then train GPT-4+ agent on these three datasets crossly using different train-test pairs and show the test performance in Figure~\ref{fig_ood}.
We observe that in most cases (2 out of 3), when the training and test data come from the same domain, agent training leads to the best test performance compared with training using other domains, where the results are intuitive to us.
However, we observed an exception when testing on algebra. Using intermediate algebra for training led to better performance than using algebra (67.5\% vs. 65.0\%). This could be because intermediate algebra shares a similar distribution with algebra, and the more harder problems in intermediate algebra could be easier to learn basic and general functions that works for basic problems.
Another observation is that using geometry as the training domain leads to the worst test performance in algebra and intermediate algebra. This is because its distribution is far from both of the other two data types.

\subsubsection{Extend to Large Scale Training Data - Batch Training}
\label{sec:batch_train}

The proposed agent training method has one obvious bottleneck, which is that the training data size is limited to the context limit of the LLM-backed optimizer. This limitation prevents the full utilization of large-scale training data. A similar bottleneck occurs in traditional model training, where the constraint is from the GPU/CPU memory. To resolve this problem, traditional machine learning uses the concept of \emph{batch training}~\cite{masters2018revisiting}. This method divides the dataset into smaller subsets (batches) and trains the model iteratively on each batch to overcome the memory limitation.

Building on this practice, we propose a straightforward batch training method for our agent training flow. Specifically, we randomly sample one batch of training data within the LLM context limit at each training iteration from large-scale training data. Other procedures remain the same. 
We evaluate the Intermediate algebra of the MATH dataset on GPT-4+ agent system with 100 problems for training and 80 problems for testing where the test data is the same as it is in previous sections. 
We tried four different batch sizes (5, 10, 15, and 20), and set the epoch to 40, 20, 13, and 10, respectively, to ensure that the number of examples used for training is the same. We show the final test performance in Figure~\ref{batch_abla}.
The results show that large training data does not necessarily lead to test performance improvement in most cases, and only one case achieved a mirror improvement. Even when the batch size is set to 20, which is the same as the training data size in Figure~\ref{batch_abla}, the test performance drops by 7.8\%. This drop may be due to the frequent changing of training examples at each epoch, which prevents the AgentOptimizer from generating stable and effective functions.

\begin{table*}[t]
\centering % This centers the table in the text
\setlength{\tabcolsep}{17pt} % Reduce column spacing
\scalebox{0.85}{
\renewcommand{\arraystretch}{0.95} 
\begin{tabular}{c|c} 
\toprule[1.5pt]
\textbf{Tasks} & \textbf{Top Used Functions} \\ 
\midrule
\multirow{2}{*}{MATH} & evaluate\_expression, calculate\_polynomial\_roots, solve\_algebraic\_equation, calculate\_circumference \\
 & calculate\_polynomial\_roots, solve\_algebraic\_equation, calculate\_complex\_magnitude \\
\midrule
\multirow{1}{*}{GAIA} & scrape\_wikipedia\_table, extract\_pdf\_text, perform\_web\_search, fetch\_web\_content \\
\midrule
\multirow{2}{*}{TabMWP} & calculate\_total\_cost, analyze\_stem\_leaf\_plot, calculate\_basic\_statistics, perform\_table\_calculation\\  
& perform\_arithmetic\_operations, statistical\_analysis  \\
\bottomrule[1.5pt]
\end{tabular}
}
\vskip -0.09in
\caption{For illustration purposes, we list frequently used (during testing) functions generated by AgentOptimizer in different tasks.}
\label{tab_gen_tool}
\vspace{-0.2in}
\end{table*}

\begin{figure}[!t]
    \begin{center}
        \centerline{\includegraphics[width=0.7\columnwidth]{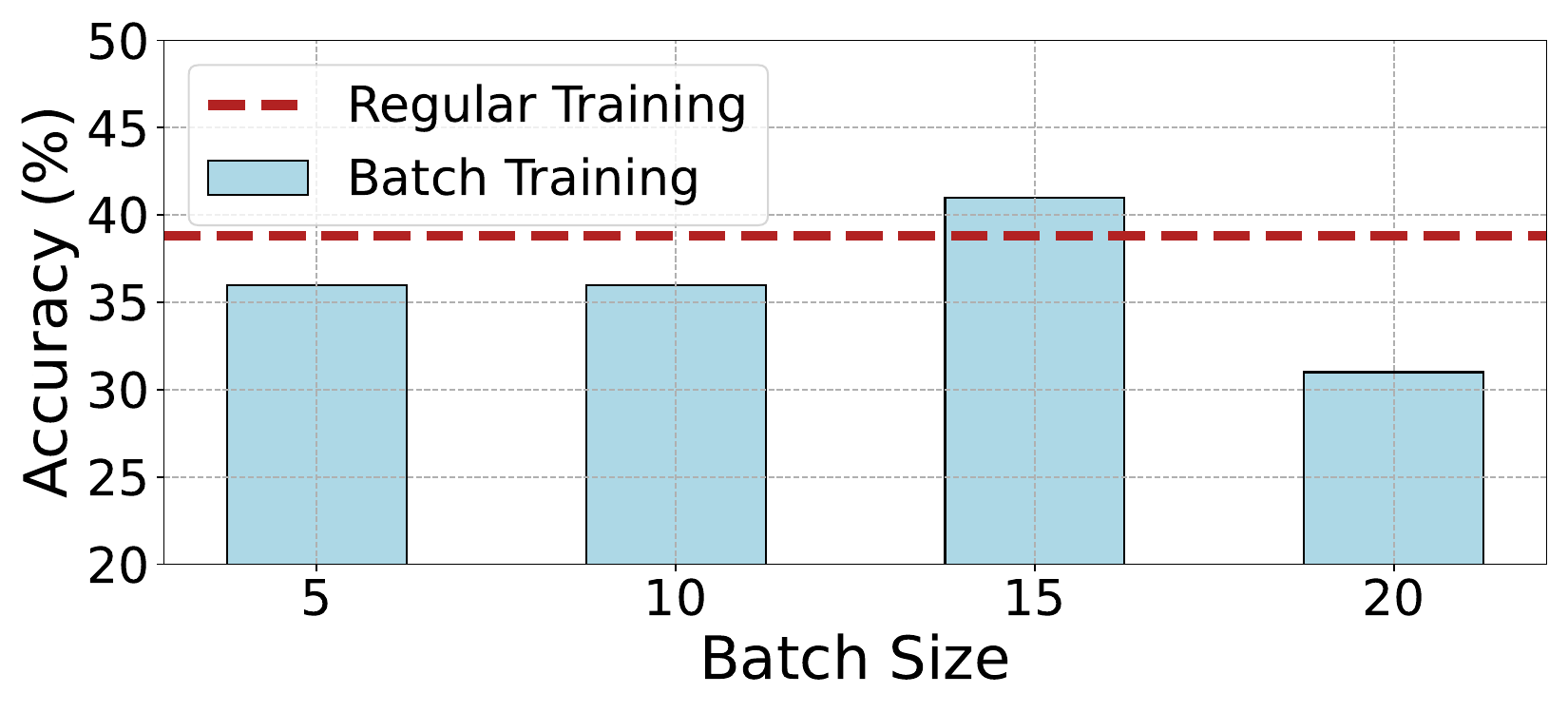}}
        \vspace{-0.2in}  % Adjust the value as needed
        \caption{The comparisons between the "regular training" of our method and the extended "batch training". 
        The batch training with an enlarged training set doesn't necessarily lead to better performance in different batch settings.
        }
        \label{batch_abla}
    \end{center}
    \vspace{-0.5in}
\end{figure}

\subsection{Agent Training v.s. Tool-Creation}
\label{Sec:tool}

Tool-creation algorithms ~\cite{cai2023large, qian2023creator} are to prompt LLMs to create tools that are tailored to specific tasks. 
Since the tool-creation procedure is a one-time process that does not include subsequent optimization mechanisms based on training performance, the design philosophy emphasizes that the created tools \emph{can be used} (without error), but not \emph{used effectively} (improve performance).

In this section, we compare the trained GPT-4+/ReAct agents and two latest tool-creation methods,  CREATOR~\cite{qian2023creator} and CRAFT~\cite{yuan2023craft}, on MATH and TabMWP datasets. 
For TabMWP, we follow the same experimental setting as Section~\ref{section_main_result}.
We choose these two datasets because the baseline codes on these two datasets are available and we can make a rigorous comparison.
To cover all data types of the MATH dataset, we randomly sample 20 examples for training and 80 examples for testing from all data types. 
As shown in Table~\ref{compare_with_tools}, after agent training, both the GPT-4+ agent and ReAct agent exhibit better performance compared with the tool-creation method, indicating agent training is a promising paradigm to distill function/tool from advanced large language models.

\subsection{Analysis of the Learned Functions}

We conducted an in-depth analysis of the generated functions. First, we present a list of frequently used functions generated for all datasets in Table~\ref{tab_gen_tool}. Then, we show the number of successful function calls at the second and end epochs (the functions may not be the same) during model training in Table~\ref{tab_tool_analysis}. We also present the widely adopted cyclomatic complexity~\cite{mccabe1994software} of the generated functions. We calculate the complexity using the Lizard Python library and present the average complexity of tools for each task when optimizing both GPT-4+ agent and ReAct agent.

Our observations indicate that the number of successful function calls exhibits significant improvement in most datasets, indicating that the optimized functions are becoming more effective compared to the initial list. Considering function complexity, a good function should have a complexity of no more than 10. A less complex function is less prone to trigger bugs. We observed that the created functions for the three tasks exhibit relatively low complexity, indicating that the functions are reliable.

\begin{table}[!t]
\setlength{\tabcolsep}{9pt} % Reduce column  
\renewcommand{\arraystretch}{0.85} 
\begin{tabular}{cccc}
\toprule[1.5pt]
Metrics                           & MATH        & GAIA        & TabMWP      \\ \hline
Second Epoch                & 11          & 8           & 19          \\
Last Epoch                  & \textbf{23} & \textbf{10} & \textbf{41} \\
Avg. Complexity &     1.2         &   3.7   &    5.0      \\ \bottomrule[1.5pt]
\end{tabular}
\vskip -0.09in
\caption{The number of successful function calls in the \textbf{second epoch} and the \textbf{last epoch} (functions may not be the same) of the agent training. We also show the cyclomatic complexity of the generated functions in the last row.
}
\label{tab_tool_analysis}
\vspace{-0.2in}
\end{table}
\section{Related Work}

There has been a growing volume of research focusing on employing LLMs to construct autonomous agents for reasoning, planning, and adapting to new observations in real-world tasks~\cite{xi2023rise, wang2023survey, hong2023metagpt, yao2022react, wu2023autogen, li2023camel, babyagi, park2023generative}. 
In such LLM agents, functions/tools/actions that LLM can leverage to interact with the environment or solve sub-tasks play a critical role, yet are often manually crafted~\cite{yao2022react}.
Recent works have explored automatic tool creation~\cite{cai2023large, qian2023creator,yuan2023craft}.
Specifically, Tool-maker~\cite{cai2023large} proposes to create tools through three demonstrations and then validates the created tool using three validation examples; CREATOR~\cite{qian2023creator} proposes to create tools exclusive for each query; 
And CRAFT~\cite{yuan2023craft} first creates customizable tools tailored for specific problems and then retrieves relevant tools for user query inference time.
In this work, we propose a conceptual framework that treats functions as learnable parameters in traditional AI models and develop a generic agent training paradigm to improve functions iteratively across epochs.
Different from prior works, our AgentOptimizer updates the function set based on the LLM agent's execution history of the whole training set, rather than making functions according to individual query-answer pair(s); this approach not only includes the specific LLM agent's behavior into consideration for function creation (in contrast to looking at the query-answer pair only), but also tends to make generic functions that work for the whole training set. By formulating an iterative optimization process, the AgentOptimizer can continuously update the functions based on the execution history of each epoch during training in a trial-and-error manner.

Sharing a similar goal of improving LLM agents, another line of work aims to enhance agent capability by modifying the underlying LLMs~\cite{patil2023gorilla,qin2023toolllm,zeng2023agenttuning}.
For instance, ToolLLM~\cite{qin2023toolllm} collects a massive amount of APIs to construct instruction data to finetune LLaMA~\cite{touvron2023llama} to obtain a new LLM optimized for using the collected APIs;
AgentTune~\cite{zeng2023agenttuning} proposes to enhance the
agent abilities through a hybrid instruction-tuning strategy to tune the LLMs parameters.
In contrast, we explore a new paradigm of training LLM agents without modifying the underlying LLM, which is particularly useful when the LLMs are online services and not available for tuning like GPT-4 or when tuning and maintaining a new LLM are expensive and time-consuming.

Besides, in this work, we leverage the exceptional capability of the LLM to build an optimizer (the AgentOptimizer) for training the agents, mimicking the numeric optimizers in model training such as SGD and Adam. Such an idea of using LLM as an optimizer has been proven effective by prior work~\cite{yang2023large, zhang2023using}. 
While these prior works mainly leverage LLM as an optimizer for optimization problems like prompt optimization~\cite{yang2023large} and hyperparameter optimization~\cite{zhang2023using}, our AgentOptimizer is particularly designed for the novel agent training paradigm and progressively update LLM agent's functions via multiple add, revise, and/or remove actions within each optimization step.

\section{Conclusion}

In this study, we propose a novel approach to train specialized LLM agents. 
The core idea is to draw an analogy between LLM agent training and traditional model training, where the learnable parameters in traditional models correspond to the operational functions of LLM agents, and the models' loss functions correspond to the historical performance metrics of the agents. 
Leveraging the impressive optimization capability of LLMs, we enhance the agents by updating the agent functions through the proposed AgentOptimizer. 
We evaluate the proposed method on multiple distinct tasks in training two typical agent systems and demonstrate that the agent training exhibits obvious performance improvement.

\clearpage
\section*{Impact Statements}
This paper presents research aimed at advancing the field of language agents. Our work has several potential societal consequences, both positive and negative, that we feel need to be highlighted.
On the positive side, language agents could be the core of many real-life applications~\cite{hosseini2023exploratory, cai2019hello}, and our work could greatly benefit these applications by enhancing the agents. For instance, it could be the core of an industrial robot~\cite{zeng2023large}, and our work could potentially enhance working efficiency.
On the negative side, the development of language agents raises the possibility of negative use of enhanced agents, such as using language agents to generate misinformation or harmful content~\cite{navigli2023biases} in social media for illegal purposes. Another concern is allowing language models to make changes in external environments~\cite{tian2023evil}. For instance, allowing language models to perform code execution in the computer may lead to unintended consequences~\cite{liu2024empirical}.

\bibliography{reference}
\bibliographystyle{icml2024}

\onecolumn
\etocdepthtag.toc{mtappendix}
\etocsettagdepth{mtchapter}{none}
\etocsettagdepth{mtappendix}{subsection}

\renewcommand{\contentsname}{Appendix}
\tableofcontents

\clearpage
\appendix

\section{Supplementary Theoretical Analysis}
In this section, we attempt to provide a theoretical analysis of the proposed agent training method. The objective is to provide an upper bound for the expected test loss difference between the trained agent function and the global optimal function. 
As an initial attempt, our analysis on the generalization bound of the agent training requires the following two strong assumptions. We leave the relaxation of these two assumptions to future work. 
\begin{assumption} \label{assumption_1}
In the agent training scenario, the training data $\mathcal{D}_{train}$ and test data $\mathcal{D}_{test}$ come from the same distribution $\mathbb{P}$, i.e., \(\mathcal{D}_{train}, \mathcal{D}_{test} \in \mathbb{P}.\)
\end{assumption}
In classical machine learning model training, it is a common practice to assume that the distribution of the training and test data are the same or similar, then use training loss as the primary metric for parameters selection.

\begin{assumption} \label{assumption_2}
Given training data $\mathcal{D}_{train}$, the proposed agent training method could identify the function set $\hat{\mathcal{F}}$ which achieves the smallest loss in $\mathcal{D}_{train}$ after agent training.
\begin{equation}
    \hat{\mathcal{F}}  = \argmin_{\mathcal{F}\subset\mathcal{V}}Loss(S_{\mathcal{F}}, \mathcal{D}_{train}).
\end{equation}
\end{assumption}

\begin{lemma} \label{lemma:difference_between_empirical_and_expected}
   Under Assumption~\ref{assumption_1}, for any agent system $S_{\mathcal{F}}$ with function set $\mathcal{F}$, with probability at least \(1-\delta\) (\(\delta \in (0,1)\)), we have:
   \[
   |Loss(S_{\mathcal{F}}, \mathcal{D}_{train}) - E[Loss(S_{\mathcal{F}}, \mathcal{D}_{test})] | \leq \sqrt{\frac{\beta \ln(1/\delta)}{2|\mathcal{D}_{train}|}},
   \]
   in which \(\beta\) represents the distance between the largest and the lowest loss value on any data instance. Specifically,  for any data instance $d \in \mathbb{P}$, $l_{\mathcal{S_{\mathcal{F}}}}(d) <\beta$, where $l_{S_{\mathcal{F}}}$ denotes the loss function, which measures the loss of each data instance for agent system  $S_{\mathcal{F}}$.
\end{lemma}

\begin{proof}[Proof of Lemma~\ref{lemma:difference_between_empirical_and_expected}]

For any training data set $\mathcal{D}_{train}$ and potential test data set $\mathcal{D}_{test}$ from the data distribution \(\mathbb{P}\), we have

\begin{equation}
|Loss(S_{\mathcal{F}}, \mathcal{D}_{train}) - E[Loss(S_{\mathcal{F}}, \mathcal{D}_{test})] | = |\frac{1}{|\mathcal{D}_{train}|}\sum_{i=1}^{|\mathcal{D}_{train}|}l_{S_{\mathcal{F}}}(d_{i}) - E_{d \sim \mathbb{P}}[l_{S_{\mathcal{F}}}(d)]|.
\end{equation}  

According to Hoeffding’s inequality~\cite{hoeffding1994probability}, we have:
\begin{align}
\label{bound_sup_0}
P(|Loss(S_{\mathcal{F}}, \mathcal{D}_{train}) - E[Loss(S_{\mathcal{F}}, \mathcal{D}_{test})] |> \epsilon ) &= P(|\frac{1}{|\mathcal{D}_{train}|}\sum_{i=1}^{|\mathcal{D}_{train}|}l_{S_{\mathcal{F}}}(d_{i}) - E_{d \sim \mathbb{P}}[l_{S_{\mathcal{F}}}(d)]| > \epsilon) \\ \nonumber
& \leq 2\exp{\frac{-2|\mathcal{D}_{train}| \epsilon^2}{\frac{1}{|\mathcal{D}_{train}|}\sum_{i=1}^{|\mathcal{D}_{train}|} \beta}} = 2\exp{\frac{-2|\mathcal{D}_{train}| \epsilon^2}{\beta}}. 
\end{align}

Then with probability at least \(1- 2\exp{\frac{-2|\mathcal{D}_{train}| \epsilon^2}{\beta}}\), we have:

\begin{equation}\label{lemma1_1}
|Loss(S_{\mathcal{F}}, \mathcal{D}_{train}) - E[Loss(S_{\mathcal{F}}, \mathcal{D}_{test})] | \leq  \epsilon .
\end{equation}  

Taking $\delta = 2\exp{\frac{-2|\mathcal{D}_{train}| \epsilon^2}{\beta}}$, we have:

\begin{equation}\label{lemma1_2}
\epsilon = \sqrt{\frac{\beta \ln (2/\delta) }{2|\mathcal{D}_{train}|}}
\end{equation}  

Combining Equation~\ref{lemma1_1} and Equation~\ref{lemma1_2}, with probability at least $1-\delta$, we have:
\begin{equation}
|Loss(S_{\mathcal{F}}, \mathcal{D}_{train}) - E[Loss(S_{\mathcal{F}}, \mathcal{D}_{test})] | \leq  \sqrt{\frac{\beta \ln (2/\delta) }{2|\mathcal{D}_{train}|}}.
\end{equation}
Which completes the proof.
\end{proof}

\begin{theorem}\label{theorem:diff_expfunc_optimalfunc}
Under Assumption~\ref{assumption_1} and Assumption~\ref{assumption_2}, with probability at least \(1-\delta\) (\(\delta \in (0,1)\)), the trained agent system $S_{\mathcal{\hat{F}}}$ with trained functionl list $\mathcal{\hat{F}}$ satisfies:
\begin{equation}
E[Loss(S_{\mathcal{\hat{F}}}, \mathcal{D}_{test})] - E[Loss(S_{\mathcal{F^*}}, \mathcal{D}_{test})] \leq 2\sqrt{\frac{\beta \ln (2/\delta) }{2|\mathcal{D}_{train}|}},
\end{equation}
\end{theorem}
where $\mathcal{F^*}$ denotes the optimal function in the function space $\mathcal{V}$, i.e., $\mathcal{F^{*}} = \argmin_{\mathcal{F}\subset\mathcal{V}}E[Loss(S_{\mathcal{F}}, \mathcal{D}_{test})]$.

\begin{proof}[Proof of Theorem~\ref{theorem:diff_expfunc_optimalfunc}]

Taking $\mathcal{\hat{F}}$ into Lemma~\ref{lemma:difference_between_empirical_and_expected}, with probability at least \(1-\delta\) (\(\delta \in (0,1)\)), we have:
\begin{equation} \label{inte:2-1}
|Loss(S_{\mathcal{\hat{F}}}, \mathcal{D}_{train}) - E[Loss(S_{\mathcal{\hat{F}}}, \mathcal{D}_{test})]| \leq \sqrt{\frac{\beta \ln (2/\delta) }{2|\mathcal{D}_{train}|}}.
\end{equation}
Considering $\hat{\mathcal{F}} = \argmin_{\mathcal{F}\subset\mathcal{V}}Loss(S_{\mathcal{F}}, \mathcal{D}_{train})$, we have:
\begin{equation} \label{inte:2-2}
Loss(S_{\mathcal{\hat{F}}}, \mathcal{D}_{train}) < Loss(S_{\mathcal{F^*}}, \mathcal{D}_{train}).
\end{equation}
Combing Equation~\ref{inte:2-1} and Equation~\ref{inte:2-2}, we have:
\begin{equation} \label{inte:2-3}
E[Loss(S_{\mathcal{\hat{F}}}, \mathcal{D}_{test})]  \leq Loss(S_{\mathcal{F^*}}, \mathcal{D}_{train}) + \sqrt{\frac{\beta \ln (2/\delta) }{2|\mathcal{D}_{train}|}}.
\end{equation}
Taking $S_{\mathcal{F^*}}$ into Lemma~\ref{lemma:difference_between_empirical_and_expected}, we have:
\begin{equation}\label{inte:2-4}
|Loss(S_{\mathcal{F^*}}, \mathcal{D}_{train}) -E[Loss(S_{\mathcal{F^*}}, \mathcal{D}_{test})] | \leq     \sqrt{\frac{\beta \ln (2/\delta) }{2|\mathcal{D}_{train}|}}.
\end{equation}

Combining Equation~\ref{inte:2-3} and Equation~\ref{inte:2-4}, with probability at least $1-\delta$, we have:
\begin{equation} \label{inte:2-4}
E[Loss(S_{\mathcal{\hat{F}}}, \mathcal{D}_{test})] - E[Loss(S_{\mathcal{F^*}}, \mathcal{D}_{test})] \leq 2*\sqrt{\frac{\beta \ln (2/\delta) }{2|\mathcal{D}_{train}|}},
\end{equation}
which completes the proof.
\end{proof}

Theorem~\ref{theorem:diff_expfunc_optimalfunc} provides an upper bound on the expected test loss difference between the trained agent function $\mathcal{\hat{F}}$ and the global optimal function $\mathcal{F^*}$. We observe from Equation~\ref{inte:2-4} that a larger training set could lead to a narrower upper bound. However, the training set is limited by the LLM’s context limit. This limitation inspires us to investigate a better way of extending the training dataset, rather than relying on the straightforward batch training approach described in Section \ref{sec:batch_train}.

\section{Supplementary Experimental Results}

\subsection{Evaluations on Other Language Models}

\begin{table*}[htb]
\centering
\setlength{\tabcolsep}{10pt} 
\begin{tabular}{cccc}
\toprule[1.5pt]
               & Code-Llama-34B & Mixtral-8x7B & GPT-3.5-turbo-1106  \\ \hline
Before Training & 7.5\%   &  23.8\%  &  25.0\%                                           \\ \hline
After Training & \textbf{11.3\%}     & \textbf{28.8\%} &  \textbf{28.8\%}                               \\ \bottomrule[1.5pt]
\end{tabular}
\caption{
The performance of agents backed by other language models is evaluated before and after agent training on the MATH dataset. The results indicate that agent training still leads to significant performance improvements.
}
\label{diff_models}
\end{table*}

In this section, we conducted experiments to evaluate the performance of agents backed by various language models after agent training, including GPT-3.5-turbo-1106~\cite{openai2022gpt}, and open-source models Mixtral-8x7B~\cite{jiang2024mixtral,jiang2023mistral} and Code-Llama-34B~\cite{roziere2023code,jayaseelanllama}. The LLM that backed the AgentOptimizer was GPT-4-1106-preview.

We performed experiments on the MATH dataset using the same settings as described in Section \ref{Sec:tool}. The results are presented in Table~\ref{diff_models}.
Our findings indicate that agent training leads to better performance on all three models, demonstrating that agent training is agnostic to the LLMs that backed the agent.

\subsection{More Experimental Results after Removing Roll-back \& Early-stop}

\begin{figure}[htb]
    \begin{subfigure}{0.38\linewidth}
    \centering
    \includegraphics[width=\linewidth]{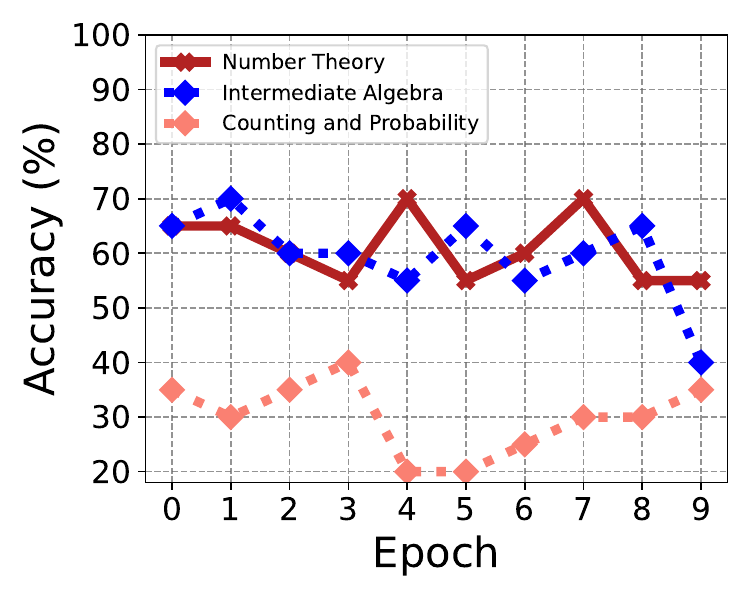}
    \centering
    \caption{Training performance w/o roll-back \& early-stop} \label{abla_rollback_more_1}
    \end{subfigure}%
    \centering
    \begin{subfigure}{0.38\linewidth}
    \centering
    \includegraphics[width=\linewidth]{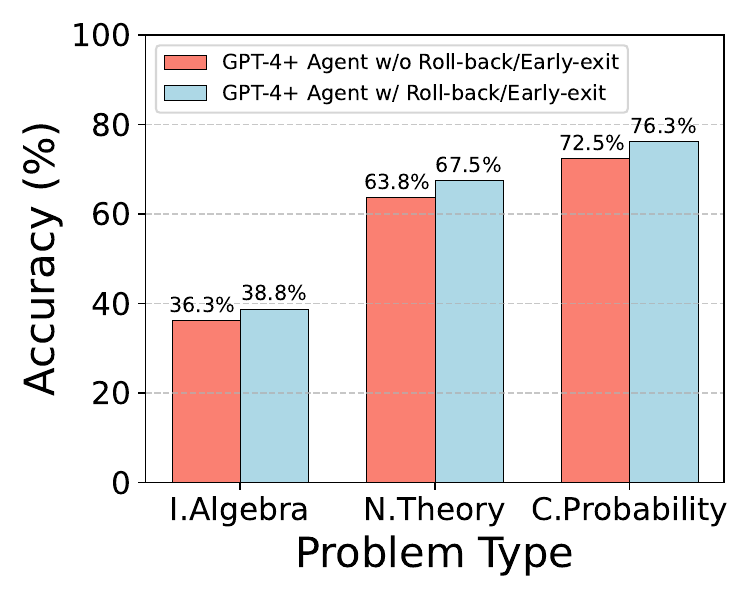} 
    \centering
    \caption{Test performance w/o roll-back \& early-stop}
    \label{abla_rollback_more_2}
    \end{subfigure}%
    \caption{After removing the roll-back and early-exit mechanisms, the learning curve of the training performance and the final test performance of GPT-4+ Agents.}
    \label{abla_rollback_more} 
\end{figure}

We present additional experimental results in Figure~\ref{abla_rollback_more} after removing roll-back and early-stop. Specifically, we further illustrate the training performance curve in Figure~\ref{abla_rollback_more_1}. We observe that the training performance fluctuated with the number of training epochs, indicating that the learned functions are not stable and may not necessarily lead to improved training performance at each epoch. This unstable function optimization leads to a drop in test performance, as shown in Figure~\ref{abla_rollback_more_2}.

\section{Supplementary Analysis of Agent Training versus Model Training}

\begin{table*}[!htb]
\centering
\setlength{\tabcolsep}{10pt} % Reduce column spacing
\begin{tabular}{ccccc}
\toprule[1.5pt]
               & Optimizer & Target & Human Interpretable & Access to Model/LLM Weights \\ \hline
Model Training & SGD etc.   &  Model Weights  &    \xmark                  &  \cmark                           \\ \hline
Agent Training & LLMs     & Functions &        \cmark              & \xmark                        \\ \bottomrule[1.5pt]
\end{tabular}
\caption{
Comparing Model Training and Agent Training: Model training relies on an optimizer such as SGD. It is not human-interpretable and requires access to model parameters. In contrast, agent training uses LLMs as the optimizer, which is interpretable in natural language and generated functions. Furthermore, agent training does not require access to model parameters.
}
\label{compare}
\end{table*}

Table~\ref{compare} summarizes the differences between these two training paradigms. Although both paradigms have a similar workflow of improving from training data leveraging their optimizers, they have different features. Specifically, the optimizers in traditional model training are gradient descent optimization algorithms, which update the model parameters in the opposite direction of the gradient of the loss function. However, the complex parameters updating logic is not interpretable to humans, and model training requires accessible parameters.
In contrast, the optimizers in agent training are LLMs, which prompt the update of agent functions using natural language at each optimization step. The optimization is interpretable to humans (functions and natural language), and it doesn’t require accessible parameters.

\section{Implementations Details}

\subsection{Prompt Design for AgentOptimizer}

You are a function optimizer. Your task is to maintain a list of functions for the assistant according to the existing function set and conversation history that happens between the assistant and the user. 

You can perform one of the following four actions to manipulate the function set using the functions you have: 

1. Revise one existing function (using revise\_function). 
2. Remove one existing function (using remove\_function).
3. Add one new function (using add\_function).
4. Directly return "TERMINATE" to me if no more actions are needed for the current function set.

Below are the principles that you need to follow for taking these four actions.

(1) Revise one existing function:
1. Pay more attention to the failed tasks and corresponding error information, and optimize the function used in these tasks according to the conversation history if needed.
2. A failed function call can occur due to incorrect input arguments (missing arguments) or an incorrect function code implementation. You should focus more on the function code implementation and make it easy to get success function call.
3. Do not revise the function that you think works well and plays a critical role in solving the problems according to the conversation history. Only making revisions if needed.
4. Sometimes, a NameError may occur. To fix this error, you can either revise the name of the function in the code implementation or revise the name of the function call to make these two names consistent.

(2) Remove one existing function:
1. Only remove the function that you think is not needed anymore in future tasks.

(3) Add one new function:
1. The added function should be general enough to be used in future tasks. For instance, if you encounter a problem that this function can solve, or one step of it, you can use the generated function directly instead of starting from scratch 
2. The added new function should solve a higher-level question that encompasses the original query and extend the code's functionality to make it more versatile and widely applicable.  
3. Replace specific strings or variable names with general variables to enhance the tool's applicability to various queries. All names used inside the function should be passed in as arguments. 
Below is an example of a function that potentially deserves to be added, which can be used to solve a higher-level question:

\begin{lstlisting}[language={python}]
{{
    "name": "evaluate_expression",
    "description": "Evaluate arithmetic or mathematical expressions provided as strings.",
    "arguments": {{
        "expression": {{
            "type": "string",
            "description": "The mathematical expression to evaluate."
        }}
    }},
    "packages": "sympy",
    "code": "from sympy import sympify, SympifyError\n\n def evaluate_expression(expression):\n    try:\n        result = sympify(expression)\n        if result.is_number:\n            result = float(result)\n        else:\n            result = str(result)\n        return result\n    except SympifyError as e:\n        return str(e)"
}}
\end{lstlisting}
(4) Directly return "TERMINATE":
If you think there is no need to perform any other actions for the current function set since the current list is optimal more actions will harm the performance in future tasks. Please directly reply to me with "TERMINATE".

One function signature includes the following five elements:
1. Function name 
2. Function description 
3. JSON schema of arguments encoded as a string
4. A list of package names imported by the function packages 
5. The code implementation

Below are the signatures of the current functions.

List A: \textcolor{burntorange}{{\{current\_function\_signature\}}}

The success rate (performance) with this function set is \textcolor{burntorange}{{\{success\_rate\}}}.
The following list are the function signatures that you have after taking \textcolor{burntorange}{{\{actions\_num\}}} actions in our previous conversations.

List B: \textcolor{burntorange}{{\{updated\_function\_signature\}}}.

We also provide more examples for different functions and their corresponding success rates. The following function signatures are arranged in are arranged in ascending order based on their success rates, where higher success rates indicate better quality.

\textcolor{burntorange}{{\{historical\_fail\_functions\}}}

Here are\textcolor{burntorange}{{ \{conversation\_num\}}} conversation histories of solving \textcolor{burntorange}{{\{conversation\_num\}}} tasks.

History:
\textcolor{burntorange}{{\{history\}}}

The following table shows the statistical information for solving each task in each conversation and indicates whether each task was successfully solved. 
1 represents correct. 0 represents wrong.

statistic: 
\textcolor{burntorange}{{\{statistic\}}}

According to the information I provide, please take one of four actions to manipulate list B using the functions you know. 
Instead of returning TERMINATE directly or taking no action, you should try your best to optimize the function set. Only take no action if you really think the current list is optimal, as more actions will harm performance in future tasks. 
Even adding a general function that can substitute the assistant’s repeated suggestions of Python code with the same functionality could also be helpful.

\subsection{Prompt Design for ReAct}

Answer the following question using your coding skills. Below is a list of the tools you can use and their detailed descriptions:

\textcolor{burntorange}{{\{tool\_descriptions\}}}

You should always follow the below template, when you respond you should provide one (Thought, Action, Action Input) triplet and wait for observation before proceeding to the next round, unless you have reached a FINAL ANSWER.

YOUR FINAL ANSWER should be a number OR as few words as possible OR a comma separated list of numbers and/or strings.

If you are asked for a number, don't use comma to write your number neither use units such as \$ or percent sign unless specified otherwise.
If you are asked for a string, don't use articles, neither abbreviations (e.g. for cities), and write the digits in plain text unless specified otherwise.
If you are asked for a comma separated list, apply the above rules depending of whether the element to be put in the list is a number or a string.

TEMPLATE:

Question: the input question you must answer

Thought: your reasoning about the current situation

Action 1: the action to take, should be one of \textcolor{burntorange}{{[\{tool\_names\}]}}

Action 1 Input: the arguments passed to action 1

Observation 1: the result of action 1

Action 2: the action to take, should be one of \textcolor{burntorange}{{[\{tool\_names\}]}}

Action 2 Input: the input to action 2

... (this Thought/Action/Action Input/Observation can repeat N times)

Thought: I now know the final answer

FINAL ANSWER: the final answer to the original input question

\subsection{Function calls of LLM backed AgentOptimizer}
\label{optimizer-actions}

\paragraph{Add\_function:} add a new function that may be used in future tasks.

\begin{lstlisting}[language=Python]
    ADD_FUNC = {
        "type": "function",
        "function": {
            "name": "add_function",
            "description": "Add a function in the context of the conversation. Necessary Python packages must be declared. The name of the function MUST be the same with the function name in the code you generated.",
            "parameters": {
                "type": "object",
                "properties": {
                    "name": {
                        "type": "string",
                        "description": "The name of the function in the code implementation."
                    },
                    "description": {
                        "type": "string",
                        "description": "A short description of the function."
                    },
                    "arguments": {
                        "type": "string",
                        "description": "JSON schema of arguments encoded as a string. Please note that the JSON schema only supports specific types including string, integer, object, array, boolean. (do not have float type) For example: { \"url\": { \"type\": \"string\", \"description\": \"The URL\", }}. Please avoid the error 'array schema missing items' when using array type."
                    },
                    "packages": {
                        "type": "string",
                        "description": "A list of package names imported by the function, and that need to be installed with pip prior to invoking the function. This solves ModuleNotFoundError. It should be string, not list."
                    },
                    "code": {
                        "type": "string",
                        "description": "The implementation in Python. Do not include the function declaration."
                    }
                },
                "required": ["name", "description", "arguments", "packages", "code"]
            }
        }
    }

\end{lstlisting}

\paragraph{Revise\_function:} revise one existing function.

\begin{lstlisting}[language=Python]
    REVISE_FUNC = {
        "type": "function",
        "function": {
            "name": "revise_function",
            "description": "Revise a function in the context of the conversation. Necessary Python packages must be declared. The name of the function MUST be the same with the function name in the code you generated.",
            "parameters": {
                "type": "object",
                "properties": {
                    "name": {
                        "type": "string",
                        "description": "The name of the function in the code implementation."
                    },
                    "description": {
                        "type": "string",
                        "description": "A short description of the function."
                    },
                    "arguments": {
                        "type": "string",
                        "description": "JSON schema of arguments encoded as a string. Please note that the JSON schema only supports specific types including string, integer, object, array, boolean. (do not have float type) For example: { \"url\": { \"type\": \"string\", \"description\": \"The URL\", }}. Please avoid the error 'array schema missing items' when using array type."
                    },
                    "packages": {
                        "type": "string",
                        "description": "A list of package names imported by the function, and that need to be installed with pip prior to invoking the function. This solves ModuleNotFoundError. It should be string, not list."
                    },
                    "code": {
                        "type": "string",
                        "description": "The implementation in Python. Do not include the function declaration."
                    }
                },
                "required": ["name", "description", "arguments", "packages", "code"]
            }
        }
    }
\end{lstlisting}

\paragraph{Remove\_function:} remove one existing function.

\begin{lstlisting}[language=Python]
    REMOVE_FUNC = {
        "type": "function",
        "function": {
            "name": "remove_function",
            "description": "Remove one function in the context of the conversation. Once remove one function, the assistant will not use this function in future conversation.",
            "parameters": {
                "type": "object",
                "properties": {
                    "name": {
                        "type": "string",
                        "description": "The name of the function in the code implementation."
                    }
                },
                "required": ["name"]
            }
        }
    }
\end{lstlisting}

\section{Generated Functions}
Here we provide a list of typical function examples generated by AgentOptimizer for all three datasets.

\subsection{Trained Functions in MATH}

\textbf{GPT-4+ Agent}

1. evaluate\_expression:
evaluate arithmetic or mathematical expressions provided as strings.

\begin{lstlisting}[language=Python]
from sympy import sympify, SympifyError

def evaluate_expression(expression):
    try:
        result = sympify(expression)        
        if result.is_number:            
            result = float(result)        
        else:            
            result = str(result)        
            return result    
        except SympifyError as e:        
            return str(e)
\end{lstlisting}

2. calculate\_polynomial\_roots:
find all real roots of a given polynomial.

\begin{lstlisting}[language=Python]
from sympy import solve, sympify, SympifyError

def calculate_polynomial_roots(polynomial):    
    try:        
        # Parse the polynomial and solve for its roots        
        roots = solve(sympify(polynomial), dict=True)        
        # Convert roots to a standard list format        
        roots_list = [str(root) for root in roots]\n        
        return roots_list    
    except SympifyError as e:        
        return str(e)
\end{lstlisting}

3. solve\_algebraic\_equation:
revises the existing function that solves a single variable algebraic equation to handle equations more robustly with variable terms on both sides.
\begin{lstlisting}[language=Python]
from sympy import symbols, Eq, solve, parse_expr

def solve_algebraic_equation(equation, variable):
    # Create a symbolic variable
    symbol = symbols(variable)
    # Parse the equation string into a sympy expression
    left_part, right_part = equation.split('=')
    eq = Eq(parse_expr(left_part), parse_expr(right_part))
    # Solve the equation for the variable
    solution = solve(eq, symbol)
    # Return the solution
    return solution
\end{lstlisting}

\textbf{ReAct Agent}

1. convert\_decimal\_to\_fraction:
convert a decimal number to its equivalent in the simplest fractional form.

\begin{lstlisting}[language=Python]
from fractions import Fraction

def convert_decimal_to_fraction(decimal_number):
    try:
        # Convert string to a floating point number
        decimal_number = float(decimal_number)
        # Create a Fraction from the decimal number
        fraction_result = Fraction(decimal_number).limit_denominator()
        # Return the fraction as a string in the form 'numerator/denominator'
        return str(fraction_result)
    except ValueError as e:
        return str(e)
\end{lstlisting}

2. evaluate\_math\_expression:
evaluate a wide range of mathematical expressions provided as strings, including basic arithmetic, factorial, combinations, and permutations.
\begin{lstlisting}[language=Python]
from sympy import sympify, factorial, binomial

def evaluate_math_expression(expression):
    try:
        # Extend the namespace with factorial and binomial functions
        local_dict = {'factorial': factorial, 'comb': binomial}
        # Evaluate the expression using sympy's sympify function
        result = sympify(expression, locals=local_dict)
        if result.is_number:
        return float(result)
        else:
        return str(result)
    except Exception as e:
        return str(e)
\end{lstlisting}

3. get\_polynomial\_degree:
given a polynomial expression as a string, return the degree of the polynomial.
\begin{lstlisting}[language=Python]
from sympy import Poly, SympifyError

def get_polynomial_degree(expression):
try:
    # Convert the string expression into a polynomial
    poly = Poly(expression)
    # Return the degree of the polynomial
    return poly.degree()
except SympifyError as e:
    return str(e)
\end{lstlisting}

\subsection{Trained Functions in GAIA}

\textbf{GPT-4+ Agent}

1. perform\_web\_search: performs a web search using Bing Search API and returns the top search results including URLs and snippets.

\begin{lstlisting}[language=Python]
import os
import requests

def perform_web_search(query):
    subscription_key = os.environ['BING_SEARCH_V7_SUBSCRIPTION_KEY']
    endpoint = os.environ['BING_SEARCH_V7_ENDPOINT'] + '/v7.0/search'
    headers = {'Ocp-Apim-Subscription-Key': subscription_key}
    params = {'q': query, 'textDecorations': True, 'textFormat': 'HTML'}
    response = requests.get(endpoint, headers=headers, params=params)
    response.raise_for_status()
    search_results = response.json()
    top_results = [{'url': result['url'], 'snippet': result['snippet']} for result in search_results.get('webPages', {}).get('value', [])]
    return top_results
    
\end{lstlisting}

2. scrape\_wikipedia\_table:
scrapes data from a table on a Wikipedia page based on a header keyword.
\begin{lstlisting}[language=Python]
import requests
from bs4 import BeautifulSoup

def scrape_wikipedia_table(url, header_keyword):
    response = requests.get(url)
    response.raise_for_status()
    soup = BeautifulSoup(response.content, 'html.parser')
    headers = soup.find_all(['h1', 'h2', 'h3', 'h4', 'h5', 'h6'])
    data = []
    for header in headers:
        if header_keyword.lower() in header.text.lower():
            table = header.find_next_sibling('table', class_='wikitable')
            if table:
                rows = table.find_all('tr')
                for row in rows:
                    cols = row.find_all(['th', 'td'])
                    cols = [ele.text.strip() for ele in cols]
                    data.append([ele for ele in cols if ele])
                break
    return data
\end{lstlisting}

3. extract\_pdf\_text:
extracts text from a PDF file.
\begin{lstlisting}[language=Python]
import fitz  # PyMuPDF

def extract_pdf_text(file_path):
    # Open the PDF file
    with fitz.open(file_path) as pdf:
        text = ''
        # Iterate over each page
        for page_num in range(len(pdf)):
            page = pdf[page_num]
            text += page.get_text()
        return text
\end{lstlisting}

\textbf{React Agent}

1. fetch\_webpage\_content: retrieve the HTML content of a given webpage URL.

\begin{lstlisting}[language=Python]
import requests

def fetch_webpage_content(url):
    response = requests.get(url)
    response.raise_for_status()
    return response.text

\end{lstlisting}

2. fetch\_bing\_search\_results:
retrieve search results from Bing Web Search API.

\begin{lstlisting}[language=Python]
import os
import requests

def fetch_bing_search_results(query):
    subscription_key = os.environ['BING_SEARCH_V7_SUBSCRIPTION_KEY']
    endpoint = os.environ['BING_SEARCH_V7_ENDPOINT'] + "/v7.0/search"
    
    headers = {'Ocp-Apim-Subscription-Key': subscription_key}
    params = {'q': query, 'textDecorations': True, 'textFormat': 'HTML'}
    
    response = requests.get(endpoint, headers=headers, params=params)
    response.raise_for_status()
    return response.json()
\end{lstlisting}

3. extract\_text\_from\_pdf:
extracts all text from a given PDF file.

\begin{lstlisting}[language=Python]
import fitz  # PyMuPDF

def extract_text_from_pdf(file_path):
    try:
        # Open the PDF file
        with fitz.open(file_path) as pdf:
            text = ''
            # Extract text from each page in the PDF
            for page in pdf:
                text += page.get_text()
            return text
    except Exception as e:
        return f'An error occurred: {str(e)}'
\end{lstlisting}

\subsection{Trained Functions in TabMWP}

\textbf{GPT-4+ Agent}

1. perform\_arithmetic\_operations: perform basic arithmetic operations such as sum, average, maximum, minimum, difference, and rate of change on a given list of numbers.

\begin{lstlisting}[language=Python]
def perform_arithmetic_operations(numbers, operation):
    result = None
    if operation == 'sum':
        result = sum(numbers)
    elif operation == 'avg':
        result = sum(numbers) / len(numbers) if numbers else None
    elif operation == 'max':
        result = max(numbers) if numbers else None
    elif operation == 'min':
        result = min(numbers) if numbers else None
    elif operation == 'diff' and len(numbers) > 1:
        result = numbers[0] - numbers[1]
    elif operation == 'rate_of_change' and len(numbers) > 1 and numbers[1] != 0:
        result = ((numbers[0] - numbers[1]) / abs(numbers[1])) * 100
    return result
\end{lstlisting}

2. analyze\_stem\_leaf\_plot
Analyze a given stem-leaf plot to calculate the total count of values within a specified range.

\begin{lstlisting}[language=Python]

def analyze_stem_leaf_plot(stem_leaf_data, min_value, max_value):
    count = 0
    for stem, leaves in stem_leaf_data.items():
        for leaf in leaves:
            value = int(stem) * 10 + leaf
            if min_value <= value < max_value:
                count += 1
    return count
\end{lstlisting}

3. calculate\_range
Calculate the range (difference between the maximum and minimum) of a list of numbers.

\begin{lstlisting}[language=Python]
def calculate_range(numbers):
    return max(numbers) - min(numbers)
\end{lstlisting}

\textbf{React Agent}

1. calculate\_total\_cost\_general: Calculate the total cost given a unit price and quantity, supporting both the quantity as a string or an integer.

\begin{lstlisting}[language=Python]
def calculate_total_cost_general(unit_price, quantity):
    return float(unit_price) * (int(quantity) if isinstance(quantity, str) else quantity)
\end{lstlisting}
% \newpage
\section{Case Study}
We present then three case studies for the trained GPT-4+ agent on three different datasets, to identify why the well-optimized learned function leads to the correct result in each case study.
\subsection{Case Study for MATH}
\vspace{-0.3in}
\begin{figure}[H]
\centering
\includegraphics[width=0.73\textwidth]{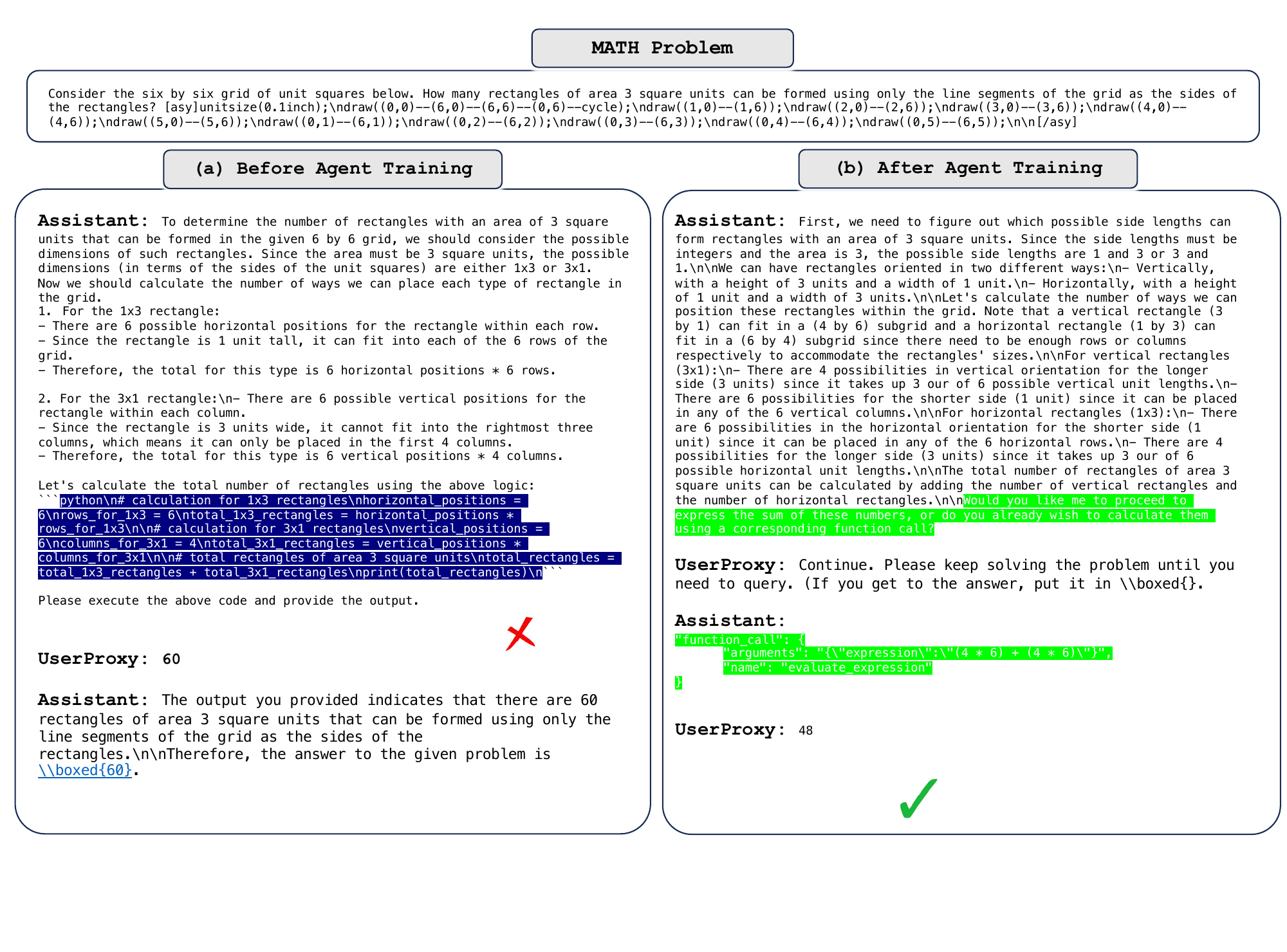}
\vspace{-0.55in}
\caption{
Comparisons of GPT-4+ Agent agents before and after agent training. After the training, the well-optimized learned function leads to correct result compared to real-time generated python code.
}
\label{case study: math}
\end{figure}
\vspace{-0.3in}
\subsection{Case Study for GAIA}
\begin{figure}[H]
\centering
\includegraphics[width=0.73\textwidth]{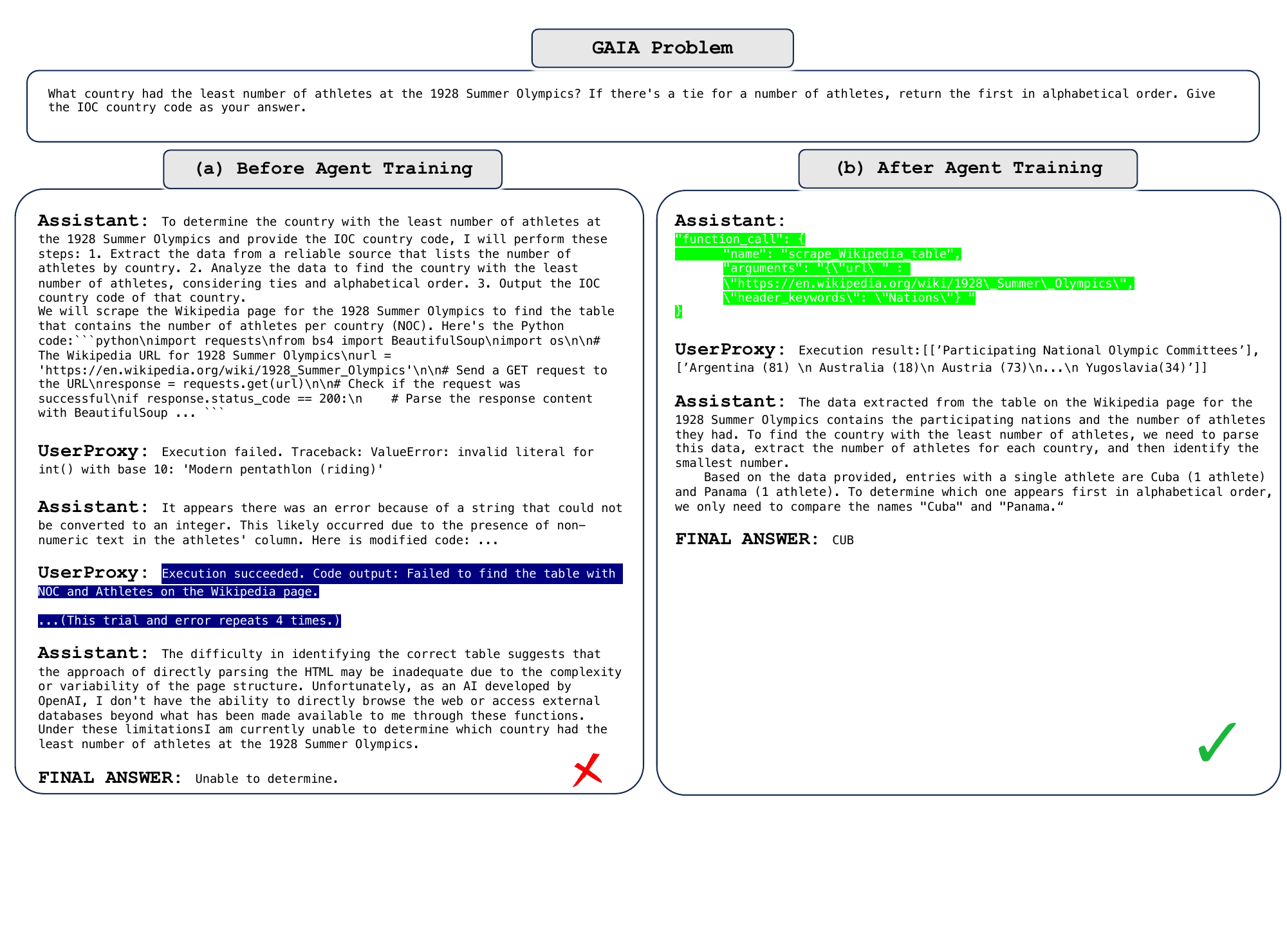}
\vspace{-0.55in}
\caption{
Comparisons of GPT-4+ Agent task-solving trajectory on GAIA before and after agent training. After training, the agent can successfully leverage the well-optimized function to handle complex web-scraping problems and solve the otherwise coding-heavy task. The suggested python code by assistant is truncated for simplicity.
}
\label{case study: math}
\end{figure}

\subsection{Case Study for TabMWP}
\begin{figure}[H]
\vspace{-0.18in}
\centering
\includegraphics[width=0.9\textwidth]{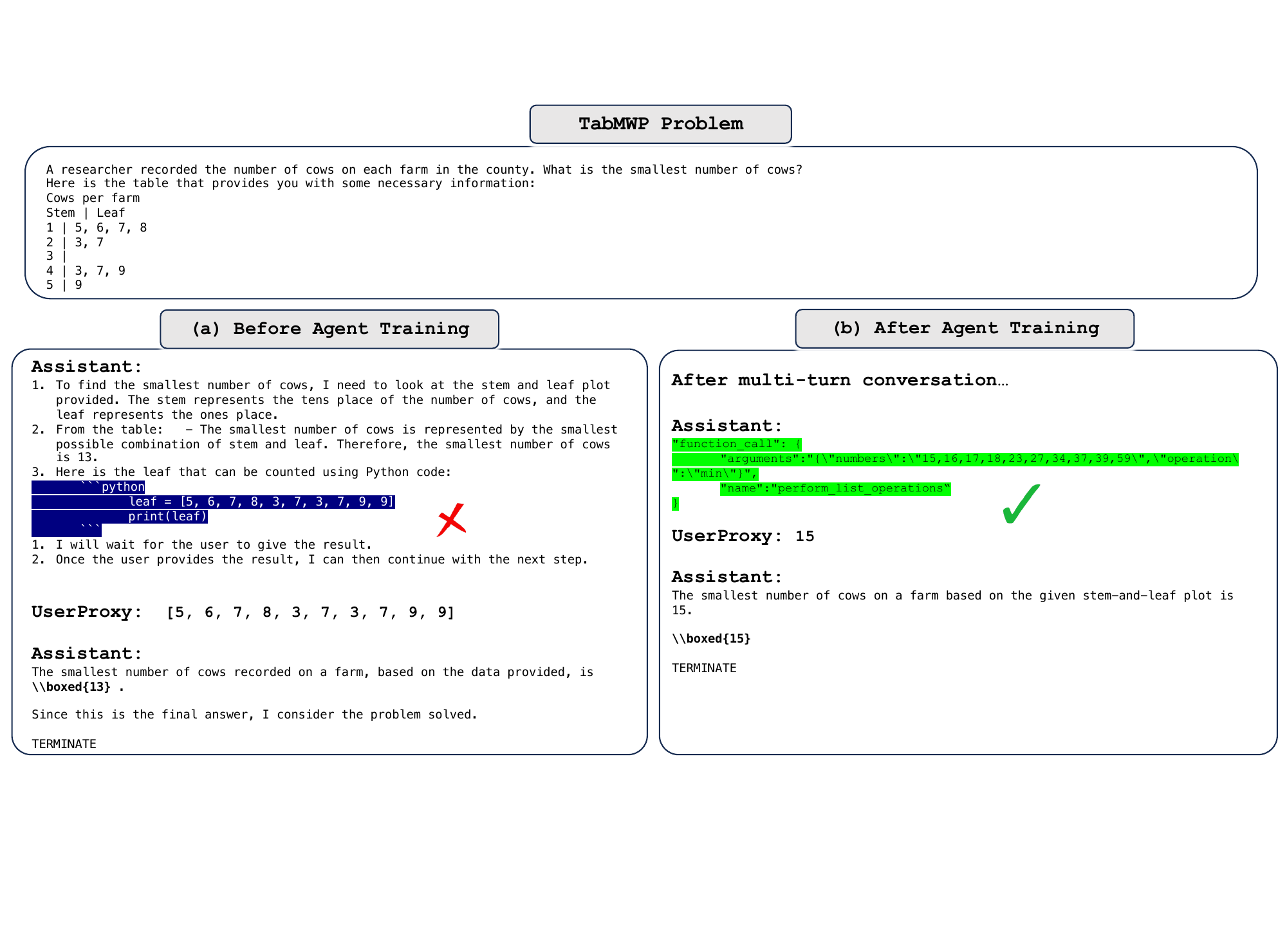}
\vspace{-0.8in}
\caption{
Comparisons of GPT-3.5-turbo + Agent task-solving trajectory on TabMWP before and after agent training. After training, the agent can successfully leverage the well-optimized function to obtain an accurate result.
}
\label{case study: math}
\end{figure}

\section{Hyperparameters Settings}
The proposed agent training method involves several hyperparameters, including the training epoch, early-stop threshold, and maximum number of actions.
In our empirical experiments across all three datasets, we consistently utilized the same hyperparameter configuration for the proposed agent training algorithm. Specifically:
(1) We set the training epoch to 10 for all experiments.
(2) An early stopping criterion was established with a threshold of 10 epochs. If there were 10 consecutive epochs without any improvement in training performance, the training process terminated.
(3) Additionally, we restricted the maximum number of actions taken during each function update step to 3.
It is essential to recognize that optimal hyperparameter settings can vary based on the specific problem and task. However, for our research, we kept these parameters fixed to ensure a consistent experimental setup. 
Combining our algorithm with hyperparameter tuning techniques from previous work~\cite{li2018hyperband, wu2021frugal, zhang2022targeted} may further enhance performance.

\section{Limitations}

A significant bottleneck in the agent training algorithm arises from that the size of training data is limited by the LLM context limit.
This constraint severely restricts its applicability to large-scale training scenarios.
Furthermore, in Section~\ref{sec:batch_train}, we empirically demonstrate that directly applying batch training techniques from traditional machine learning to agent training is ineffective and presents a non-trivial challenge. We regard addressing the limitations as our follow-up work.

\end{document}